\documentclass{article}

\usepackage{microtype}
\usepackage{graphicx}
\usepackage{subfigure}
\usepackage{booktabs} 

\usepackage{hyperref}

\usepackage{caption}
\usepackage[accepted]{icml2023}

\usepackage{amsmath}
\usepackage{amssymb}
\usepackage{mathtools}
\usepackage{amsthm}

\DeclareMathOperator*{\argmin}{arg\,min}
\newcommand{\norm}[1]{\left\lVert#1\right\rVert}
\newcommand{\lz}{\mathsf{LZ}}

\usepackage[capitalize,noabbrev]{cleveref}

\theoremstyle{plain}
\newtheorem{theorem}{Theorem}[section]
\newtheorem{proposition}[theorem]{Proposition}
\newtheorem{lemma}[theorem]{Lemma}

\theoremstyle{definition}

\theoremstyle{remark}

\newtheorem*{proposition*}{Proposition}

\usepackage[textsize=tiny]{todonotes}

\icmltitlerunning{Maximum Likelihood Learning of Unnormalized Models for Simulation-Based Inference}

\begin{document}

\twocolumn[
\icmltitle{Maximum Likelihood Learning of Unnormalized Models \\ for Simulation-Based Inference}



\icmlsetsymbol{equal}{*}

\begin{icmlauthorlist}
\icmlauthor{Pierre Glaser}{yyy}
\icmlauthor{Michael Arbel}{xxx}
\icmlauthor{Samo Hromadka}{yyy}
\icmlauthor{Arnaud Doucet}{zzz}
\icmlauthor{Arthur Gretton}{yyy}
\end{icmlauthorlist}

\icmlaffiliation{yyy}{Gatsby Computational Neuroscience Unit, University College London, London, UK}
\icmlaffiliation{xxx}{Universit\'{e} Grenoble Alpes, CNRS, Grenoble INP, LJK, 38000 Grenoble, France}
\icmlaffiliation{zzz}{Deepmind}

\icmlcorrespondingauthor{Pierre Glaser}{pierreglaser@gmail.com}

\icmlkeywords{Machine Learning, ICML}

\vskip 0.3in
]



\printAffiliationsAndNotice{} 

\begin{abstract}
We introduce two synthetic likelihood methods for Simulation-Based Inference (SBI), to conduct either amortized or targeted inference from experimental observations when a high-fidelity simulator is available.
Both methods learn a conditional
energy-based model (EBM) of the likelihood using synthetic data generated by the
simulator, conditioned on parameters drawn from a proposal distribution.
The learned likelihood can then be combined with any prior to obtain a posterior estimate,
from  which samples can be drawn using MCMC.
Our methods uniquely combine a flexible Energy-Based Model and the minimization of a KL loss: this is in contrast to other synthetic likelihood methods, which either rely on normalizing flows, or minimize score-based objectives; choices that come with known pitfalls.
We demonstrate the properties of both methods on a range of synthetic datasets, and apply them to a neuroscience model of the pyloric network in the crab, where our method outperforms prior art for a fraction of the simulation budget.

\end{abstract}
\vspace{-2em}

\section{Introduction}


Simulation-based modeling expresses a system as a \emph{probabilistic program}
\citep{ghahramani2015probabilistic},
which describes, in a mechanistic manner, how samples from the system are drawn given the
parameters of the said system. This probabilistic program can be concretely implemented in a computer
 - as a \emph{simulator} - from which synthetic parameter-samples pairs can be
drawn. This setting is common in many scientific and engineering disciplines such as stellar events in cosmology \citep{Alsing:2018,Schafer:2012},
particle collisions in a particle accelerator for high energy physics \citep{Eberl:2003,Sjostrand:2008}, 
and biological neural networks in neuroscience \citep{Markram:2015,Pospischil:2008}.
Describing such systems using a probabilistic program often turns out to be easier than specifying
the underlying probabilistic model with a tractable probability distribution. We consider the task of \emph{inference} for such systems, which consists in computing the posterior
distribution of the parameters given observed (non-synthetic) data.
When a likelihood function of the simulator is available alongside with a prior belief
on the parameters, inferring the posterior distribution of the parameters given data is
possible using Bayes' rule. Traditional
inference methods such as variational techniques \citep{Wainwright:2008} or Markov Chain
Monte Carlo \citep{andrieu2003introduction} can then be used to produce approximate posterior
samples of the parameters that are likely to have generated the observed data.
Unfortunately, the likelihood function of a simulator is computationally intractable in
general, thus making the direct application of traditional inference techniques
unusable for simulation-based modelling.

\emph{Simulation-Based Inference} (SBI) methods \citep{cranmer_frontier_2020} are
methods specifically designed to perform inference in the setting of a simulator with
an intractable likelihood. These methods repeatedly generate synthetic data using the
simulator to build an estimate of the posterior, that either can be used for any observed data
(resulting in a so-called \emph{amortized} inference procedure) or one that is \emph{targeted} for a specific observation.
While the accuracy of inference
increases as more simulations are run, so does  computational cost, especially when
the simulator is expensive, which is common in many physics applications
\citep{cranmer_frontier_2020}. In high-dimensional settings, early simulation-based inference techniques such as Approximate Bayesian Computation
(ABC) \citep{Marin:2012}
struggle to generate high quality posterior samples at a reasonable cost, since ABC repeatedly rejects simulations that fail to reproduce the observed data \citep{beaumont2002approximate}. More recently, model-based
inference methods \citep{wood2010statistical,papamakarios_2019_sequential,Hermans:2020, Greenberg:2019}, which encode
information about the simulator via a parametric density (-ratio) estimator of the data, have been
shown to drastically reduce the number of simulations
needed to reach a given inference precision \citep{lueckmann2021benchmarking}.
The computational gains are particularly important when comparing ABC to \emph{targeted} SBI methods,
implemented in a \emph{multi-round} procedure that refines the model around the
observed data, by sequentially simulating data points that are closer to the observed
ones \citep{Greenberg:2019,papamakarios_2019_sequential,Hermans:2020}.
	

Previous model-based SBI methods have used their parametric estimator to learn the likelihood (e.g. the conditional density 
specifying the probability of an observation being simulated given a specific parameter set, \citealt{wood2010statistical, papamakarios_2019_sequential, pacchiardi2020score}), the likelihood-to-marginal ratio \citep{Hermans:2020},
or the posterior function directly \citep{Greenberg:2019}. We focus in this paper on likelihood-based (also called Synthetic Likelihood; SL, in short) methods, of which two main instances exist: (Sequential) Neural Likelihood Estimation (or (S)NLE) \citep{papamakarios_2019_sequential}, which learns a likelihood estimate using a normalizing flow trained by optimizing a Maximum Likelihood (ML) loss; and Score Matched Neural Likelihood Estimation (SMNLE \citealt{pacchiardi2020score}), which learns an unnormalized (or \emph{Energy-Based}, \citealt{LeCun:2006}) likelihood model  trained using conditional score matching. Recently, SNLE was applied successfully to challenging neural data \citep{Deistler2021dis}. However, limitations still remain in the approaches taken by both (S)NLE and SMNLE. On the one hand, flow-based models may need to use very complex architectures to properly approximate distributions with rich structure such as multi-modality  \citep{kong20expressive,
Cornish:2020}. On the other hand, score matching, the objective of SMNLE, minimizes the Fisher Divergence between the data
and the model, a divergence that fails to
capture important features of probability distributions such as mode proportions \citep{wenliang2020blindness,zhang2022towards}. This is unlike Maximimum-Likelihood based-objectives, whose maximizers satisfy attractive theoretical properties \citep{bickel2015mathematical}.

{\bfseries Contributions.} In this work,  we introduce \emph{Amortized Unnormalized Likelihood Neural Estimation} (AUNLE), and \emph{Sequential UNLE}, a pair of SBI Synthetic Likelihood methods performing respectively sequential and targeted inference. Both methods learn a Conditional Energy Based Model of the simulator's likelihood
using a Maximum Likelihood (ML) objective, and perform MCMC on the posterior estimate obtained after invoking Bayes' Rule.
While posteriors arising from conditional EBMs exhibit a particular form of intractability called \emph{double intractability}, which requires the use of tailored MCMC techniques for inference, we train AUNLE using a new approach which we call \emph{tilting}. This approach automatically removes this intractability in the final posterior estimate, making AUNLE compatible with standard MCMC methods, and significantly reducing the computational burden of inference. Our second method, SUNLE, departs from AUNLE by using a new training technique for conditional EBMs which is suited when the proposal distribution is not analytically available. While SUNLE returns a doubly intractable posterior, we show that inference can be carried out accurately through robust implementations of doubly-intractable MCMC or variational methods. We demonstrate the properties of AUNLE and SUNLE on an array of synthetic benchmark models \citep{lueckmann2021benchmarking},
and apply SUNLE to a neuroscience model of the crab \emph{Cancer borealis}, where we improve the posterior accuracy over prior state-of-the-art while needing only a fraction of the simulations required by the most efficient previous method \citep{glockler2021variational}.

\begin{figure}[htbp]
    \centering
    \includegraphics[width=0.3\textwidth]{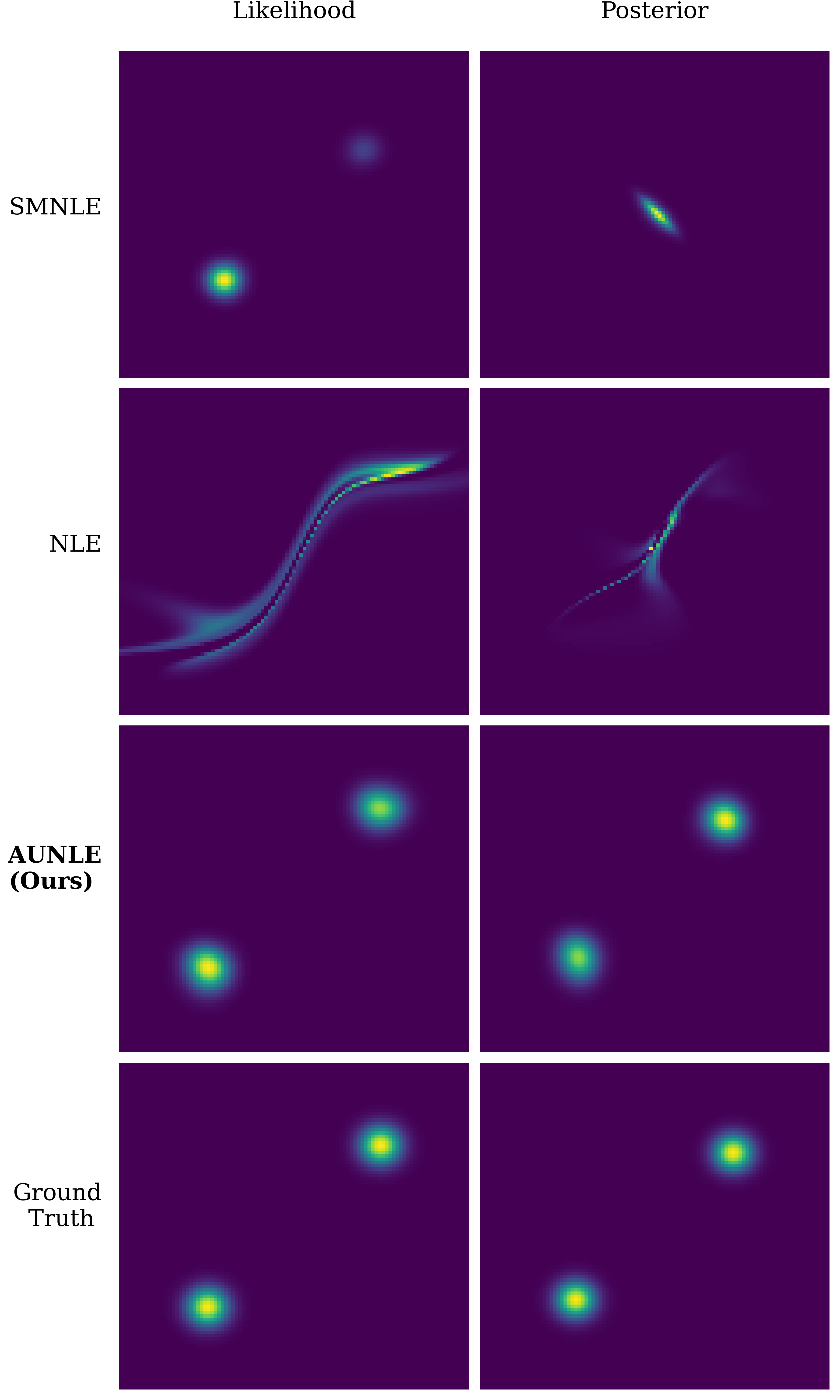}
    \caption{Performance of SMNLE, NLE and AUNLE, all trained using a simulator with a bimodal likelihood $p(x|\theta)$ and a Gaussian prior $p(\theta)$, using 1000 samples.
    Top: simulator likelihood $p(x | \theta_0)$  for some fixed $\theta_0$. Bottom: posterior estimate.
    }
    \label{fig:multimodal-posterior}
\end{figure}

\vspace{-1.5em}
\section{Background}\label{sec:background}

Simulation Based Inference (SBI) refers to the set of methods aimed at estimating the
posterior $p(\theta|x_o)$ of some unobserved parameters $\theta \in \Theta \subset \mathbb R^{d_\Theta}$ given some
observed variable $ x_o \in \mathcal X \subset \mathbb R^{d_{\mathcal X}} $ recorded from a physical system, and a prior $
p(\theta) $. In SBI, one assumes access to a simulator $G: (\theta, u)  \longmapsto y = G(\theta,
u)$, from which samples $y|\theta$ can be drawn, and whose associated
likelihood $p(y|\theta)$ accurately matches the likelihood $p(x|\theta)$ of the
physical system of interest.
Here, $u$ represents draws of all random variables involved in performing draws of
$x|\theta$. By a slight abuse of notation, we will not distinguish between the \emph{physical} random variable $x$
representing data from the \emph{physical} system of interest, and the \emph{simulated} random variable 
$y$ drawn from the simulator: we will use $x$ for both.
The complexity of the simulator \citep{cranmer_frontier_2020} prevents access to a simple form for the
likelihood $p(x|\theta)$, making standard Bayesian inference impossible. Instead,
SBI methods perform inference by drawing parameters from a proposal distribution
$\pi(\theta)$, and use these parameters as inputs to the simulator $G$ to obtain a set of simulated pairs $ (x, \theta)$
which they use to compute a posterior estimate of $p(\theta |x) $.
Specific SBI submethods have been designed to handle separately the case of
\emph{amortized} inference, where the practitioner seeks to obtain
a posterior estimate valid for any $x_o$ (which might not
be known a priori), and \emph{targeted} inference, where the posterior
estimate should maximize accuracy for a specific observed variable
$x_o$. Amortized inference methods often simply set their proposal distribution $\pi$
to be the prior $p$, whereas targeted inference methods iteratively refine their proposal $\pi$
to focus their simulated observations around the targeted $x_o $ through a \emph{sequence} of
simulation-training rounds \citep{papamakarios_2019_sequential}.

\vspace{-0.5em}
\subsection{(Conditional) Energy-Based Models}
Energy-Based Models \citep{LeCun:2006} are unnormalized probabilistic models of the form
\begin{equation*} \label{}
\begin{aligned}
	q_{\psi}(x) = \frac{e^{-E_{\psi}(x)}}{Z(\psi)}, \quad Z(\psi) = \int_{  }^{  } e^{-E_{\psi}(x)} \text{d}x,
\end{aligned}
\end{equation*}
where $Z(\psi)$ is the intractable normalizing constant of the model, and $ E_{\psi} $ is called the \emph{energy function}, usually set to be a neural network with weights $ \psi $.
By directly modelling the density $p(x)$ of the data through a flexible energy function, simple EBMs can capture rich geometries and multi-modality, whereas other model classes such a normalizing flows may require a more complex architecture \citep{Cornish:2020}.
 The flexibility of EBMs comes at the cost of having an intractable density $q_\psi(x)$ due to the presence of the normalizer $Z(\psi)$, increasing the challenge of both training and sampling. In particular, an EBM's log-likelihood $ \log q_{\psi} $ and
associated gradient $ \nabla_{ \psi }  \log q_{\psi}  $ both contain terms involving the (intractable) normalizer $Z(\psi)$:
\vspace{-1em}
\begin{equation} \label{eq:EBM-likelihood-and-gradient}
\begin{aligned}
    \log q_{\psi}(x) &= -E_{\psi}(x) - \overbrace{\log Z(\psi)}\limits^{\text{intractable}}, \\
    \quad \nabla_{ \psi } \log q_{\psi}(x) &=
    - \nabla_{ \psi } E_{\psi}(x) + \underbrace{\mathbb{E}_{x \sim q_{\psi}} \nabla_{ \psi }
    E_{\psi}(x)}\limits_{\text{intractable}},
\end{aligned}
\end{equation}
making \emph{exact} training of EBMs via Maximum Likelihood impossible.
Approximate likelihood optimization can be performed using a Gradient-Based algorithm where at each iteration $k$, the intractable expectation (under the EBM $ q_{\psi_k} $) present in $ \nabla_{ \psi }  \log  q_{\psi_k} $ is replaced by  a \emph{particle approximation} $ \widehat{ q } = \frac{1}{N}\sum_{ i=1 }^{ N } w_i \delta_{y_{i}} $ of $ q_{\psi} $. The particles $ y^{(i)} $ forming $ \widehat{ q } $ are traditionally set to be samples from a MCMC chain with invariant distribution $ q_{\psi_k} $, with uniform weights $ w_i = \frac{1}{N} $; recent work on EBM for high-dimensional image data uses an adaptation of Langevin Dynamics \citep{raginsky2017non,du_implicit_2020,nijkamp2019learning,kelly2021no}. We outline the traditional ML learning procedure for EBMs in \cref{alg:ebm-ml-training} (Appendix), where $\texttt{make\_particle\_approx}(q, \hat{q}_0)$ is a generic routine producing a particle approximation of a target unnormalized density $q$ and an initial particle approximation $\hat{q}_0$.


Energy-Based Models are naturally extended to both \emph{joint} EBMs  $q_\psi(\theta, x) = \frac{e^{-E_\psi(\theta, x)}}{Z(\psi)}$ \citep{kelly2021no, grathwohl2020your} and \emph{conditional} EBMs (CEBMs \citealt{khemakhem2020ice,pacchiardi2020score}) of the form:
\begin{equation} \label{eq:conditional-ebm}
\begin{aligned}
	q_\psi(x|\theta) =\frac{e^{-E_\psi(x, \theta)}}{Z(\theta, \psi)}, \quad Z(\theta; \psi) = \int_{  }^{  } e^{-E_{\psi}(x, \theta)} \text{d}x.
\end{aligned}
\end{equation}
Unlike joint and standard EBMs, conditional EBMs define a family of conditional
densities $ q_{\psi}(x|\theta) $, each of which is endowed with an intractable
normalizer $Z(\theta, \psi)$. 

\vspace{-0.5em}
\subsection{Synthetic Likelihood Methods for SBI}
Synthetic Likelihood (SL) methods \citep{wood2010statistical, papamakarios_2019_sequential,
pacchiardi2020score} form a class of SBI methods that learn a conditional density model $q_{\psi}(x |
\theta)$ of the unknown likelihood $ p(x|\theta) $ for every possible pair of
observations and parameters $ (x, \theta) $. The set $ \{q_{\psi}(x|\theta),\,\,
\psi \in \Psi \} $ is a model class parameterised by some vector $ \psi \in \Psi $,
which recent methods set to be a neural network with weights $ \psi $. We describe the existing Neural SL variants to date.

{\bfseries Neural Likelihood Estimation} (NLE, \citealt{papamakarios_2019_sequential}) sets $ q_{\psi} $ to a (normalized) flow-based model, and is optimized by maximizing the \emph{average conditional log-likelihood} $\mathbb E_{\pi(\theta) p(x|\theta)} \log q_{\psi}(x|\theta)$.
		 NLE performs inference by invoking Bayes' rule to obtain an unnormalized posterior estimate $p_{\psi}(\theta|x) = \frac{ q_{\psi}(x|\theta) p(\theta) }{ \int_{  }^{  }
q_{\psi}(x|\theta)p(\theta)\text{d}\theta } \propto p(\theta) q_{\psi}(x|\theta)$
	from which samples can be drawn either using MCMC, or Variational Inference \citep{glockler2021variational}. 
		
{\bfseries Score Matched Neural Likelihood Estimation} (SMNLE, \citealt{pacchiardi2020score}) models the unknown likelihood  using a conditional Energy-Based Model $ q_{\psi}(x|\theta) $ of the form of \cref{eq:conditional-ebm}, trained using a score matching objective adapted for conditional density estimation.
	The use of an unnormalized likelihood model makes the posterior estimate obtained via Bayes' Rule known up to a $\theta$-dependent term:
\begin{equation} \label{eq:doubly-intractable-posterior}
\begin{aligned}
	q_{\psi}(\theta|x) &\propto  p(\theta) q_{\psi}(x|\theta) \propto \frac{ e^{-E_\psi(x, \theta)} p(\theta) }{\underbrace{ Z(\theta,\psi) }\limits_{\text{intractable}}}, \\
 \quad Z(\theta, \psi) &= \int_{  }^{  } e^{-E_{\psi}(x, \theta)} \text{d}x.
\end{aligned}
\end{equation}
Posteriors of this form are called \emph{doubly intractable} posteriors \citep{moller2006efficient}, and can be sampled from a subclass of MCMC algorithms designed specifically to handle doubly intractable distributions \cite{murray2006doubly, moller2006efficient}.

Both the likelihood objective of NLE and the score-based objective of SMNLE do not involve the analytic expression of the proposal $ \pi $, making it easy to adapt these methods for either amortized or targeted inference.
To address the limitations of both methods mentioned in the introduction, we next propose a method that combines the use of flexible Energy-Based Models as in SMNLE, while being optimized using a likelihood loss as in NLE.

\vspace{-0.7em}
\section{Unnormalized Neural Likelihood Estimation}\label{sec:sbi-ebm}

In this section, we present our two methods, Amortized-UNLE and Sequential-UNLE.
Both AUNLE and SUNLE approximate the unknown
likelihood $ p(x|\theta) $ for any possible pair of $ (x, \theta) $ using a
\emph{conditional} Energy-Based Model $ q_{\psi}(x|\theta) $
as in \cref{eq:conditional-ebm}, where $ E_{\psi} $ is some neural network.
Additionally, AUNLE and SUNLE are both trained using a likelihood-based loss; however, the
training objectives and inference phases differ to account for the
specificities of amortized and targeted inference, as detailed below. 

\vspace{-0.5em}

\subsection{Amortized UNLE} \label{subsec:a-unle}


Given a likelihood model $ q_{\psi}(x|\theta) $, a natural learning procedure would involve fitting a model $ q_{\psi}(x|\theta) \pi(\theta) $ of the true ``joint synthetic'' distribution $ \pi(\theta)p(x|\theta) $, as  NLE does.
We show, however, that using an alternative -- tilted -- version of this model allows to compute a posterior that is more tractable than those computed by other SL methods relying on conditional EBMs such as SMNLE \citep{pacchiardi2020score}.
Our method, AUNLE, fits a joint probabilistic model $q_{\psi, \pi}$  of the form:
\begin{equation} \label{eq:UNLE-joint-model}
\begin{aligned}
q_{\psi, \pi}(x, \theta) &:= \frac{ \pi(\theta) e^{-E_{\psi}(x, \theta)} }{ Z_{\pi}(\psi) }, \\
Z_{\pi}(\psi) &= \int_{  }^{  } \pi(\theta) e^{-E_{\psi}(x, \theta)} \text{d}x \text{d}\theta.
\end{aligned}
\end{equation}
by maximizing its log-likelihood $\mathcal L_a(\psi): = \mathbb{E}_{ \pi(\theta) p(x | \theta) }\left [ \log q_{\psi, \pi}(x, \theta) \right ]  $ using an instance of \cref{alg:ebm-ml-training}. The gain in tractability offered by AUNLE is a direct consequence of the following proposition.

%

\begin{proposition}\label{prop:aunle-auto-normalization}
	Let $ \mathcal  P_{\psi}: = \left \{ q_{\psi}(\cdot|\theta) \; , \; \;
	\psi \in \Psi \right \}  $, and $q_\psi \in P_\psi$. Then we have:
	\begin{itemize}
		\item (likelihood modelling) $ q_{\psi, \pi}(x | \theta) = q_{\psi}(x|\theta)$
		\item (joint model tilting) $ q_{\psi, \pi}(x, \theta) = f(\theta) \pi(\theta) q_{\psi}(x|\theta) $, for $ f(\theta) := Z(\theta, \psi)/Z_{\pi}(\psi),$ and  $Z(\theta, \psi)$  from \eqref{eq:conditional-ebm}
		\item ((Z, $\theta$)-uniformization) If $ p(\cdot | \theta) \in \mathcal P_{\psi} $, then the ${\psi}^{\star} $ minimizing 
		$\mathcal L_a(\psi)$ satisfies:
			$ q_{\psi^\star}(x|\theta) = p(x|\theta) $, and $ Z(\theta, {\psi}^{\star}) =  Z_{\pi}( {\psi}^{\star})$.
	\end{itemize}
\end{proposition}
\begin{proof}
The first point follows by holding $\theta$ fixed in $q_{\psi, \pi}(x, \theta)$. To prove the second point, notice that $ q_{\psi, \pi}(x, \theta) = \frac{ Z(\theta, \psi) }{ Z(\theta, \psi) } \frac{ \pi(\theta) e^{-E(x, \theta)} }{ Z_\pi(\psi) } = \frac{ Z(\theta, \psi) }{ Z_{\pi}(\psi) } \pi(\theta) \frac{ e^{-E(x, \theta)} }{ Z(\theta, \psi) }$. For the last point, note that at the optimum, we have that $ q_{{\psi}^{\star}, \pi}(x, \theta) = \pi(\theta) p(x|\theta) $. Integrating out $ x $ on both sides of the equality yields $ f(\theta) \pi(\theta) = \pi(\theta) $, proving the result.
\end{proof}
\cref{prop:aunle-auto-normalization} shows that AUNLE indeed learns a
likelihood model $ q_{\psi}(x|\theta) $ through a joint model $ q_{\psi, \pi} $
\emph{tilting} the prior $ \pi $ with $ f(\theta) $.
Importantly, this tilting
guarantees that the optimal likelihood model will have a normalizing function $
Z(\theta; \psi) $ constant (or \emph{uniform}) in $ \theta $, reducing AUNLE's
posterior to a standard unnormalized posterior $q_{ {\psi}^{\star}}(\theta | x) = p(\theta)\frac{ e^{-E_{ {\psi}^{\star}}}(\theta, x) }{ Z_{\pi}( {\psi}^{\star})  }$. AUNLE then performs inference using classical MCMC algorithms targetiting $q_\psi$.
The standard nature of AUNLE's posterior contrasts with the posterior of SMNLE \cite{pacchiardi2020score}, and allows to expand the range of inference methods applicable to it, which otherwise
would have been restricted to MCMC methods for doubly intractable distributions. In particular, the sampling cost of inference could be further reduced by performing a Variational Inference step such as in \cite{glockler2021variational}.
Whether or not the $(Z,\theta)$-uniformity holds will depend on the degree to which $q_{\psi,\pi}(x|\theta)$ correctly models $p(x|\theta)$. This is particularly difficult when $p(x|\theta)$ is a ``complicated'' function of  $\theta$ (e.g. non-smooth, diverging). We further investigate this scenario when it arises in our experiments (see the SLCP model and \cref{app-sec:tilting-validation}).

\begin{algorithm}[htbp] \caption{Amortized-UNLE} \label{alg:a-unle}
	\begin{algorithmic}
		\STATE \hspace{-1em}\textbf{Input:} prior $p(\theta)$, simulator $G$, budget $ N $, initial EBM parameters $\psi_0$
		\STATE \hspace{-1em}{\bf Output:} posterior estimate $ q_{\psi}(\theta|x) $
		\STATE \hspace{-1em}{\bf Initialize} $\pi=p, q_{\psi_0, \pi} \propto e^{-E_{\psi_0}(x, \theta)} \pi(\theta)$
		\STATE \hspace{-1em}{{\bf for} $i=0,\dots,N$ {\bf do}}
		\STATE \hspace{-1em}\hspace{1em} Draw $ \theta \sim \pi $, $ x \sim G(\theta, \cdot) $
		\STATE \hspace{-1em}\hspace{1em} Add  $ (\theta, x) $ to $ \mathcal  D $
		\STATE \hspace{-1em}{{\bf end for}}
		\STATE \hspace{-1em}{{Get }}${\psi}^{\star}: =  \texttt{maximize\_ebm\_log\_l}(\mathcal D, \psi_0)$
		\STATE \hspace{-1em}Set $ q_{\psi^\star}(\theta|x):= e^{-E_{\psi^\star}(x, \theta)}p(\theta) $ 
		\STATE \hspace{-1em}Infer using MCMC on $ q_{\psi^\star}(\theta |x) $
	\end{algorithmic}
\end{algorithm}

\vspace{-0.5em}
\subsection{Targeted Inference using Sequential-UNLE} \label{subsec:s-unle}

In this section, we introduce our second method, Sequential-UNLE (or SUNLE in short),
which performs targeted inference for a specific observation $x_o$.
SUNLE follows the traditional methodology of targeted inference by splitting  the simulator budget $N$ 
over $R$ rounds (often equally), where in each round $r$, a likelihood estimate $q_{\psi^{\star}_r}(x|\theta)$ in the form of a conditional EBM is trained using all the currently available simulated data $
\mathcal  D $. This allows to construct a new posterior estimate $
q_{\psi^\star_r}(\theta|x){=} e^{-E_{\psi^{\star}_r}(x,\theta)}p(\theta)/ Z(\theta,\psi^{\star}_r)$ which is used to sample parameters $\{\theta^{(i)}\}_{i=1}^{N/R}$ that are then provided to the simulator for generating new data $x^{i}\sim G(\theta^{(i)})$. The new data are added to the set $\mathcal D$ and are expected to be more similar to the observation of interest $x_o$.
This procedure allows to focus the simulator budget on regions relevant to the single observed data
of interest $x_o$, and, as such, is expected to be more  efficient in terms of the simulator use than amortized methods \citep{lueckmann2021benchmarking}. 
Next, we discuss the learning procedure for the likelihood model and the posterior sampling. 

\vspace{-0.5em}
\subsubsection{Learning the likelihood} At each round $r$, the effective proposal $\pi$ of the training data available can be understood (provided the number of data points drawn at reach rounds is
randomized) as a mixture probability: $\pi :=  \frac{1}{r}(\pi^{(0)}(\theta) {+}
q_{\psi^\star_1}(\theta|x_o)  {+} \dots {+} q_{\psi^\star_{r-1}}(\theta|x_o))$ which is used to update the likelihood model.
In this case, the analytical form of $\pi$ is unavailable as it requires computing the normalizing constants of the posterior estimates at each round, thus making the tilting approach introduced for
AUNLE impractical in the sequential setting. 
Instead, SUNLE learns a likelihood model maximizing the \emph{average conditional log-likelihood},
\vspace{-1em}
\begin{equation} \label{eq:avg-conditional-log-l}
\begin{aligned}
	\hspace{-1em}\mathcal  L(\psi) &= \frac{1}{N}\sum\limits_{ i=1 }^{ N }
	\log q_{\psi}(x^{i} | \theta^i),
\end{aligned}
\end{equation}
where $(x^{i}, \theta^{i})_{i=1}^N$ are the current samples. Unlike standard EBM objectives, this loss \emph{directly} targets the likelihood $q_\psi(x|\theta)$, thus bypassing the need for modelling the proposal $\pi$. We propose $\texttt{maximize\_cebm\_log\_l}$  (\cref{alg:cebm-ml-training}, Appendix), a method that optimizes this objective (previously used for normalizing flows in \citealp{papamakarios_2019_sequential}) when the density estimator is a conditional EBM. The intractable term of \cref{eq:avg-conditional-log-l} is an average over the EBM probabilities conditioned on all parameters from the training set, and thus differs from the intractable term of \eqref{eq:EBM-likelihood-and-gradient}, composed of a single integral. \cref{alg:cebm-ml-training} approximates this term during training by keeping track of one particle approximation $\widehat{q}_i = \delta_{\tilde{x}_i}$ per conditional density $q_\psi(\cdot|\theta^i)$ comprised of a single particle. The algorithm proceeds by updating only a batch of size $B$ of such particles using an MCMC update with target probability
chain $q_{\psi_k}(\cdot|\theta^i)$, where $\psi_k$ is the EBM iterate at iteration $k$ of round $r$.
Learning the likelihood using \cref{alg:cebm-ml-training} allows to use all the existing simulated data during training without re-learning the proposal, maximizing sample efficiency while minimizing learning complexity. The multi-round procedure of SUNLE is summarized in \cref{alg:s-unle}.


\begin{algorithm}[htbp]\caption{Sequential-UNLE}\label{alg:s-unle}
	\begin{algorithmic}
		\STATE \hspace{-1em}\textbf{Input:} prior $p(\theta)$, simulator $G$, budget  $ N $, no. rounds $R$
		\STATE \hspace{-1em}{\bf Output:} Posterior estimate $ q_{\psi}(\theta|x) $
		\STATE \hspace{-1em}{\bf Initialize} $\pi^{(0)} = p, \psi_0^* = \psi_0, q_{\psi_0, \pi} \propto e^{-E_{\psi_0}(x, \theta)} \pi(\theta), w_0$
		\STATE \hspace{-1em}Get $ \{\theta^{(i)} \sim \pi(\theta)\}_{i=1}^{N/R}$, set $ \mathcal D =\{\theta^{(i)}, x^{(i)} \sim G(\theta, \cdot)\}_{i=1}^{N/R} $
		\STATE \hspace{-1em}{{\bf for} $r=1,\dots,R$ {\bf do}}
		\STATE \hspace{-1em}\hspace{0.4em}{{Get }}$q_{\psi^\star_r}(x|\theta) \coloneqq \texttt{maximize\_cebm\_log\_l}( \mathcal D, q_{\psi^\star_{r-1}})$
  
		\STATE \hspace{-1em}\hspace{0.4em}{{Set }}$q_{\psi^\star_r}(\theta|x) \propto p(\theta) q_{\psi^\star}(x|\theta)$
  
        \STATE \hspace{-1em}\hspace{0.4em}Get $q_{\psi_r^*}(\theta|x), \{\theta^{(i)}\}_{i=1}^{N/R}$ via Doubly-Intr. MCMC or DIVI+MCMC (Explained in \cref{eq:subsubsec:sunle-posterior-sampling})
		\STATE \hspace{-1em}\hspace{0.4em}Set $ \mathcal D = \mathcal D \cup \{\theta^{(i)}, x^{(i)} \sim  G(\theta^{(i)}, \cdot)\}_{i=1}^{N/R}$
		\STATE \hspace{-1em}{{\bf end for}}
        \STATE \hspace{-1em}{\bf Return} $q_{\psi_R^\star}(\theta|x)$
	\end{algorithmic}
\end{algorithm}
\vspace{-0.5em}
\subsubsection{Posterior sampling}\label{eq:subsubsec:sunle-posterior-sampling}

Unlike AUNLE, SUNLE's likelihood estimate $ q_{\psi^\star_R}(\cdot|\theta) $ does not inherit the $(Z, \theta)$-uniformization property guaranteed by \cref{prop:aunle-auto-normalization}. As a consequence, its posterior $ q_{ {\psi}^{\star}_R}(\theta|x) $ is \emph{doubly intractable} as it contains an intractable $\theta$-dependent term  $Z(\psi^{\star}_R,\theta)$. We discuss two methods to sample from $q_{\psi_R^\star}(\theta|x)$: Doubly Intractable MCMC, and a two-step approach which performs MCMC on a ``singly intractable'' approximation of the doubly intractable posterior.
\vspace{-2em}
\paragraph{Doubly Intractable MCMC}
Doubly Intractable MCMC  methods \cite{moller2006efficient, murray2006doubly} are MCMC algorithms that can generate samples from a doubly intractable posterior. They consist in running a standard MCMC algorithm targeting an augmented distribution $p(\theta, y_{\textrm{aux}} | x)$ whose marginal in $\theta$ equals the posterior $q_\psi(\theta|x)$: approximate posterior samples are obtained by selecting the $\theta$ component of the augmented samples returned by the MCMC algorithm while throwing away the auxiliary part. Importantly, such MCMC algorithms need to sample from the likelihood $ q_{\psi}(x|\theta) $ at every iteration to compute the acceptance probability of the proposed augmented sample.
As SUNLE's likelihood $ q_{\psi}(x|\theta) $ cannot be tractably sampled  exactly, our implementation proceeds as in \cite{pacchiardi2020score, everitt2012bayesian, alquier2016noisy} and replaces exact likelihood sampling by approximate sampling using MCMC.
\vspace{-1em}
\paragraph{Doubly Intractable Variational Inference}

While samples returned by doubly intractable MCMC algorithms often accurately estimate their target \cite{pacchiardi2020score, everitt2012bayesian, alquier2016noisy}, working with doubly intractable posteriors nonetheless complicates the task of inference: the increased computational cost arising from running an inner MCMC chain targeting $q_\psi(x|\theta)$ limits the total number of posterior samples obtainable given a reasonable time budget. Additionally the shape of pairwise conditionals \cite{glockler2021variational} $p(\theta_i, \theta_j | \theta_{k \neq i, j}, x)$, available when $p$ is a standard unnormalized posterior, becomes inaccessible in the doubly intractable case, as the normalizing function $Z(\theta)$ depends on $(\theta_i, \theta_j)$. In the following, we propose \emph{Doubly Intractable Variational Inference} (DIVI), an inference method that computes an unnormalized approximation of SUNLE's doubly intractable posterior, thus alleviating the issues discussed above.  DIVI's  posterior takes the form 
\begin{equation} \label{eq:log-z-net}
\begin{aligned}
    q_{\psi, \eta}(\theta|x) &\propto  p(\theta) e^{-E_\psi(x, \theta) - \textrm{LZ}_\eta(\theta)} \\
                             &\simeq p(\theta) e^{-E_\psi(x, \theta) - \log Z(\theta, \psi)} (\propto q_\psi(\theta|x)) 
\end{aligned}
\end{equation}
where $\textrm{LZ}_\eta(\theta)$ is a neural network with weights $\eta$. As \cref{eq:log-z-net} suggests,
$q_{\psi, \eta}$ becomes an unnormalized equivalent of $q_\psi(\theta|x)$
if and only if $\textrm{LZ}_\eta(\cdot)$ equals SUNLE's log-normalizing function $\log Z(\cdot, \psi)$ (up to an additive constant).
In \cref{prop: cond-exp-l2-opt}, we frame $\log Z(\theta, \psi)$ as the unique solution (up to an additive constant) of a specific minimization problem:
\begin{proposition} \label{prop: cond-exp-l2-opt}
Assume that $E_\psi(x, \theta)$ is differentiable w.r.t $\theta$, and let $\mathcal F$ be the space of 1-differentiable real-valued functions on $\Theta$.
Let $\nu$ be any distribution with full support on $\Theta$, and let $f^\star \in \mathcal F$. Then $f^\star$ is a solution of: 
\begin{equation*}
    \min_{f\in\mathcal{F}} \mathbb{E}_{p_\psi(x|\theta) \nu(\theta)} l(x, \theta; f), 
\end{equation*}
(where $l(x, \theta; f) := \norm{\nabla_\theta (f(\theta) + E_\psi(x, \theta))}^2$)
if and only if $f^\star = \log Z(\theta, \psi) + C$, for some constant $C$.
\end{proposition}
We provide a proof in \cref{app-subsec:lemma-proof}. \cref{prop: cond-exp-l2-opt}'s objective function takes the form of a sample average on $q_\psi(x|\theta)\nu(\theta)$ with optimal solution $\log Z(\cdot, \psi)$. DIVI, summarized in \cref{alg:divi}, leverages this fact and produces an approximation $LZ_{\eta^\star}(\cdot)$ of $\log Z(\cdot, \psi)$ by first obtaining samples $\{x^i, \theta^i\} \sim \nu(\theta)q_\psi(x|\theta)$  and returning $\textrm{LZ}_{\eta^\star}(\cdot)$, where
\begin{equation} \label{eq:m-estimation}
\begin{aligned}
    \eta^\star = \arg \min_{\eta} \frac{1}{n} \sum_{i=1}^{n} l(x^{(i)}, \theta^{(i)}, LZ_\eta)
\end{aligned}
\end{equation}
which is precisely the M-estimator of $\log Z$ associated with $\eta$'s parameter set $H$.
The training samples $\{x^{(i)}, \theta^{(i)}\}$, are computed in parallel by sampling $\{\theta^{(i)}\}$ from the proposal $\nu$, and sampling $\{x^{(i)}|\theta^{(i)}\}$  using MCMC chains targeting $q_\psi(x|\theta^{(i)})$ for each $i$. DIVI is attractive from a computational standpoint as it avoids the need to run a Doubly Intractable sampler at the cost of a standard MCMC step. On the other hand, the difficulty of the learning problem of DIVI increases with dimension of the parameter space $\Theta$. Thus, we recommend using DIVI when the parameter space is of low dimension.

\begin{algorithm}[htbp] \caption{DIVI$(\mathcal D, \psi, \eta)$}\label{alg:divi}
\label{alg:lznet-training}
	\begin{algorithmic}
		\STATE \hspace{-1em}\textbf{Input:} proposal $\nu$, doubly intractable posterior $ q_\psi(\theta|x) $, initial parameter $\eta_0$, sample size $N$
		\STATE \hspace{-1em}{\bf Output:} Standard posterior approximation $q_{\psi, \eta}$ of $q_\psi$
		\STATE \hspace{-1em}{\bf Initialize} \hspace{-0.em} $\eta_0 = \eta$, $\mathcal{E} = \{\}$
        \STATE \hspace{-1em}{{\bf for} $i=1,\dots,N$ {\bf do}}
        \STATE \hspace{-1em}\hspace{0.4em}Sample $\theta^{(i)} \sim \nu$, $x^{(i)}|\theta^{(i)}\sim q_\psi(\cdot | \theta^{(i)})$ via MCMC
        \STATE \hspace{-1em}\hspace{0.4em}Add $(\theta^{(i)}, x^{(i)})$ to $\mathcal{E}$
        \STATE \hspace{-1em}{{\bf end for}}
        \STATE \hspace{-1em} Get $\eta^\star = \arg \min \sum_{i=1}^{N} l(x^{(i)}, \theta^{(i)})$
		\vspace{0.20em}
		\STATE \hspace{-1em}{\bf Return} $q_{\psi, \eta}:= p(\theta) e^{-E_\psi(x, \theta) - \textrm{LZ}_{\eta^\star}(\theta)}$
	\end{algorithmic}
\end{algorithm}

\vspace{-0.5em}

\section{Experiments}\label{sec:experiments}

\vspace{-0.7em}

In this section, we study the performance and  properties of AUNLE and SUNLE in three different settings: a toy model that highlights the failure modes of other synthetic likelihood methods, a series of benchmark datasets for SBI, and a real life neuroscience model.

\vspace{-2em}

\paragraph{Experimental details}
AUNLE and SUNLE are implemented using \verb+jax+ \citep{frostig2018compiling}.
We approximate expectations of AUNLE's joint
EBM using 1000 independent MCMC chains with a Langevin kernel parameterised by
a step size $ \sigma $, that automatically update their step size to maintain an acceptance rate of $ 0.5 $
during a per-iteration warmup period, before freezing the chain and computing a
final particle approximation. Additionally, we introduce  a new method
which replaces the MCMC chains by a single Sequential Monte Carlo sampler \citep{chopin2020introduction,Del-Moral:2006},
which yields a similar performance as the Langevin-MCMC approach discussed above, but is more robust for lower computational budgets (see \cref{app-subsec:algorithms}). The particle approximations are persisted across iterations \citep{tieleman2008training, du_implicit_2020} to reduce the risk of learning a ``short run'' EBM \citep{nijkamp2019learning,xie2021tale} that would not approximate the true likelihood correctly (see \cref{app-subsec:short-run} for a detailed discussion). All experiments are averaged across 5 random seeds (and additionally 10 different observations $x_o$ for benchmark problems).
We provide all code\footnote{\url{https://github.com/pierreglaser/sunle}} needed to reproduce the experiments of the paper. Training and inference are computed using a single RTX5000
GPU.
For benchmark models, a single
round of EBM training takes around 2 minutes on a GPU (see \cref{app-sec:computational-cost}).


\subsection{A toy model with a multi-modal likelihood}

\vspace{-0.5em}

First, we illustrate the issues that SNLE and SMNLE can face when applied to model certain distributions using a simulator with a bi-modal likelihood. Such a likelihood is known to be hard to model by normalizing flows, which, when fitted on multi-modal data, will assign high-density values to low-density regions of the data in order to ``connect'' between the modes of the true likelihood \citep{Cornish:2020}. Moreover, multi-modal distributions are also poorly handled by score-matching, since score-matching minimizes the Fisher Divergence between the model and the data distribution, a divergence which does not account for mode proportions \citep{wenliang2020blindness}. \cref{fig:multimodal-posterior} shows the likelihood model learned by NLE and SMNLE on this simulator, which exhibit the pathologies mentioned above: the score-matched likelihood only recovers a single mode of the likelihood, while the flow-based likelihood has a distorted shape. In contrast, AUNLE estimates both the likelihood and the posterior accurately. 
This suggests that AUNLE has an advantage when working with more complex, possibly multi-modal, distributions, as we confirm later in \cref{sec:real_life}. 

%

\vspace{-0.5em}

\subsection{Results on SBI Benchmark Datasets}

\vspace{-0.5em}

We next study the performance of AUNLE and SUNLE on 4 SBI benchmark datasets with  well-defined likelihood and varying dimensionality and structure \citep{lueckmann2021benchmarking}:

{\bfseries SLCP}:  A toy SBI model introduced by \citep{papamakarios_2019_sequential} with a unimodal Gaussian likelihood $ p(x|\theta) $. The dependence of $p(x|\theta)$ on $\theta$ is nonlinear, yielding a  complex posterior.

{\bfseries The Lotka-Volterra Model} \citep{lotka1920analytical}: An ecological
model describing the evolution of the populations of two interacting species, usually
referred to as predators and prey.

{\bfseries Two Moons}: A famous 2-d toy model with posteriors comprised of two moon-shaped regions, and yet not solved completely by SBI methods.

{\bfseries Gaussian Linear Uniform}: A simple gaussian generative model, with a 10-dimensional parameter space.
\begin{figure*}[t]
    \centering
    \includegraphics[width=\textwidth]{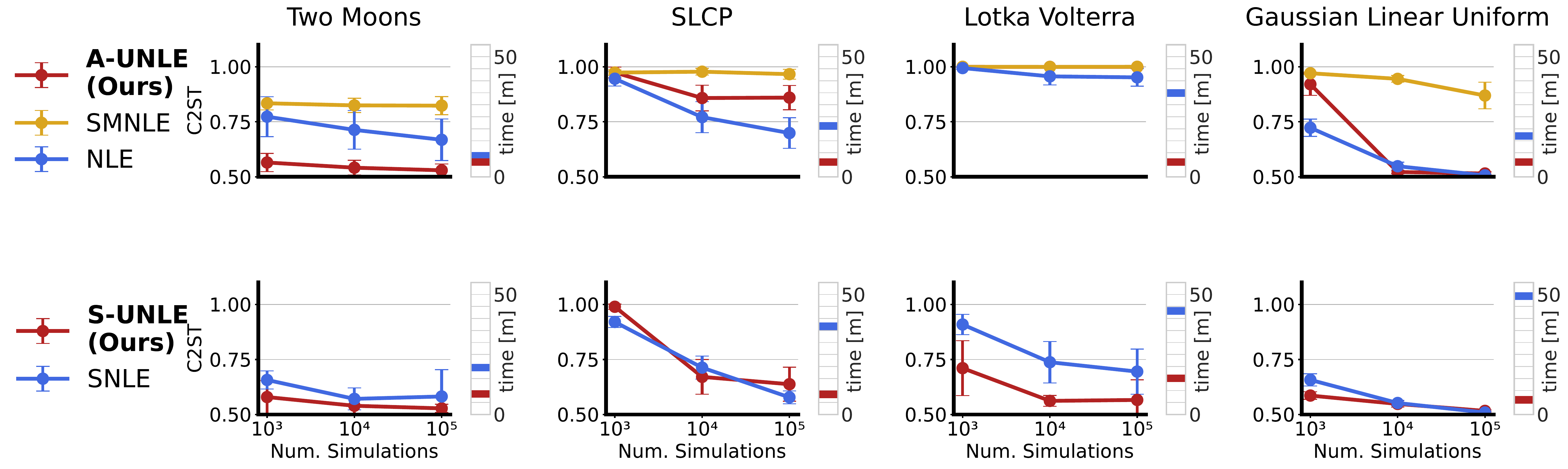} %
    \caption{Performance of AUNLE (resp. SUNLE) compared with NLE and SMNLE (resp. SNLE),
    using the Classifier Accuracy Metric \citep{lueckmann2021benchmarking}
    (lower is better). Runtime, in minutes, is also displayed for all methods except SMNLE, which was too large to display. AUNLE and SUNLE exhibit robust performance across a wide
    array of problems. Additional details on the experimental setup can be found in \cref{app-subsec:experimental-setup}.}
    \label{fig:benchmarks}
    \vspace{-1em}
\end{figure*}
These models encompass a variety of posterior structures (see \cref{app-subsec:posteriors} for posterior pairplots): the two-moons and
SLCP posteriors are multimodal, include cutoffs, and exhibit sharp and narrow
regions of high density, while posteriors of the Lotka-Volterra model place mass on
a very small region of the prior support. We compare the performance of AUNLE and SUNLE with  NLE and its sequential analogue SNLE, respectively: NLE and SNLE represent the gold standard of current synthetic likelihood methods, and perform particularly well on benchmark
problems \citep{lueckmann2021benchmarking}. 
We use the same set of hyperparameters for all models, and use a 4-layer MLP with 50 hidden units and swish activations for the energy function. Results are shown in \cref{fig:benchmarks}. All experiments used the DIVI method to obtain posterior samples at each round.

While some fluctuations exist depending on the task considered, these results show that the performance of AUNLE  (and SUNLE when targeted inference is necessary) is on par with that of (S)NLE, thus demonstrating that a generic method involving Energy-Based models can be trained
robustly, without extensive hyperparameter tuning.
Interestingly, the model where UNLE has the greatest advantage over NLE is Two Moons, which is the benchmark that exhibits a likelihood with the most complex geometry; in comparison, the three remaining benchmarks have simple normal (or log-normal) likelihood, which are unimodal distributions for which normalizing flows are particularly well suited. This point underlines the benefits of
using EBMs to fit challenging densities.

Interestingly, we notice that in the case of SLCP, SUNLE performs as well as SNLE, while AUNLE performs worse than NLE. The reason is that the likelihood of the SLCP simulator is non-smooth, and diverges to $+\infty$ at $\theta_{3, 4} = (0, 0)$. The $(Z, \theta)$-uniformity of AUNLE's optimal likelihoods $q_{\psi^\star}(x|\theta)$ makes its optimal energies $E_{\psi^\star}$ non-smooth in that case, and thus hard to estimate. In contrast, SUNLE, whose optimal likelihoods are not $(Z, \theta)$-uniform, admits smooth optimal energies for that problem, which are easier to estimate. 

%
Finally, we remark that  SMNLE, which addresses only \emph{amortized} inference \cite{pacchiardi2020score} struggled in practice for the toy problems investigated here.

\vspace{-0.8em}

\begin{figure}[htbp]
    \centering
    \subfigure{\includegraphics[width=0.4\columnwidth]{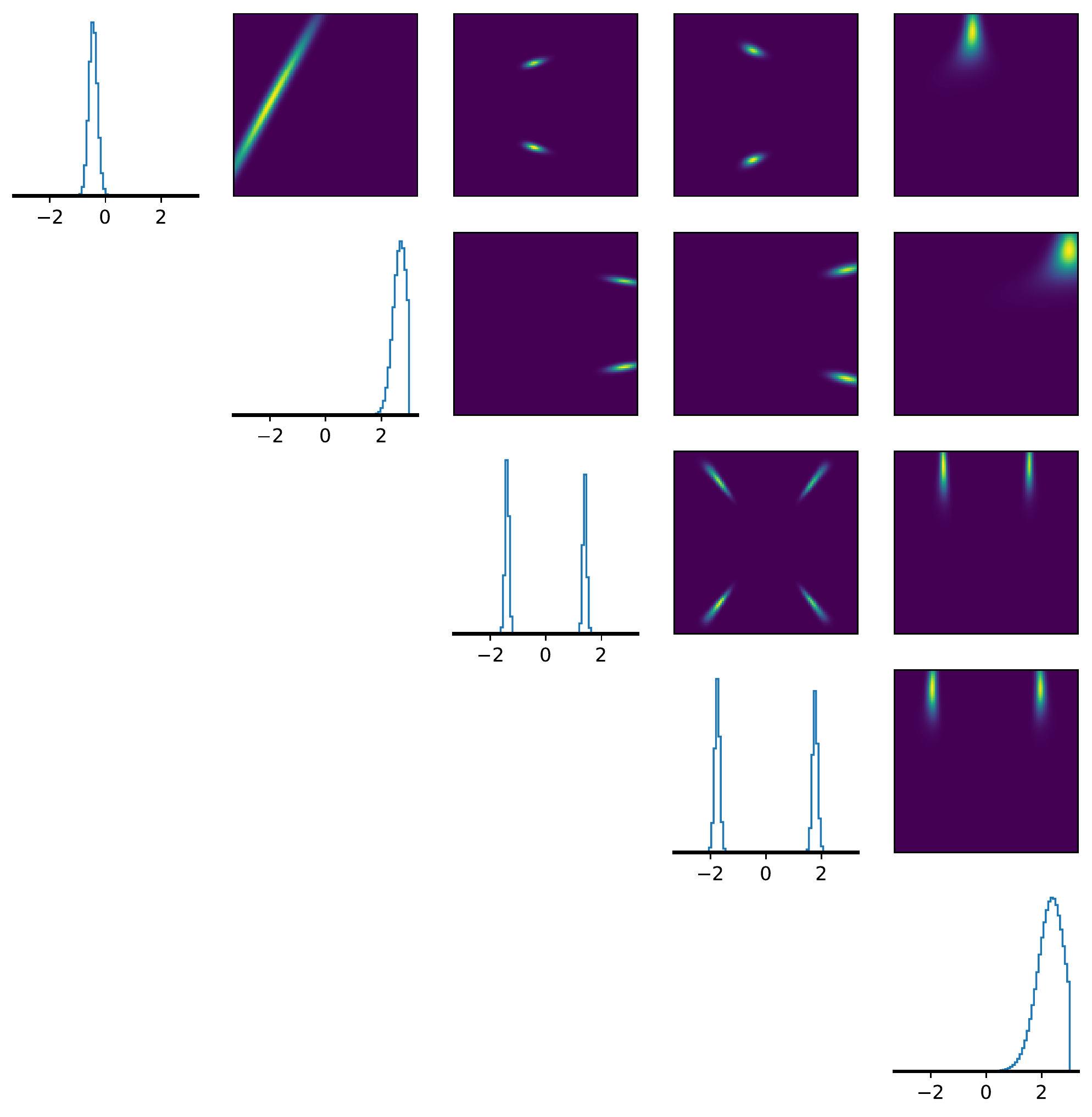}}
    \subfigure{\includegraphics[width=0.4\columnwidth]{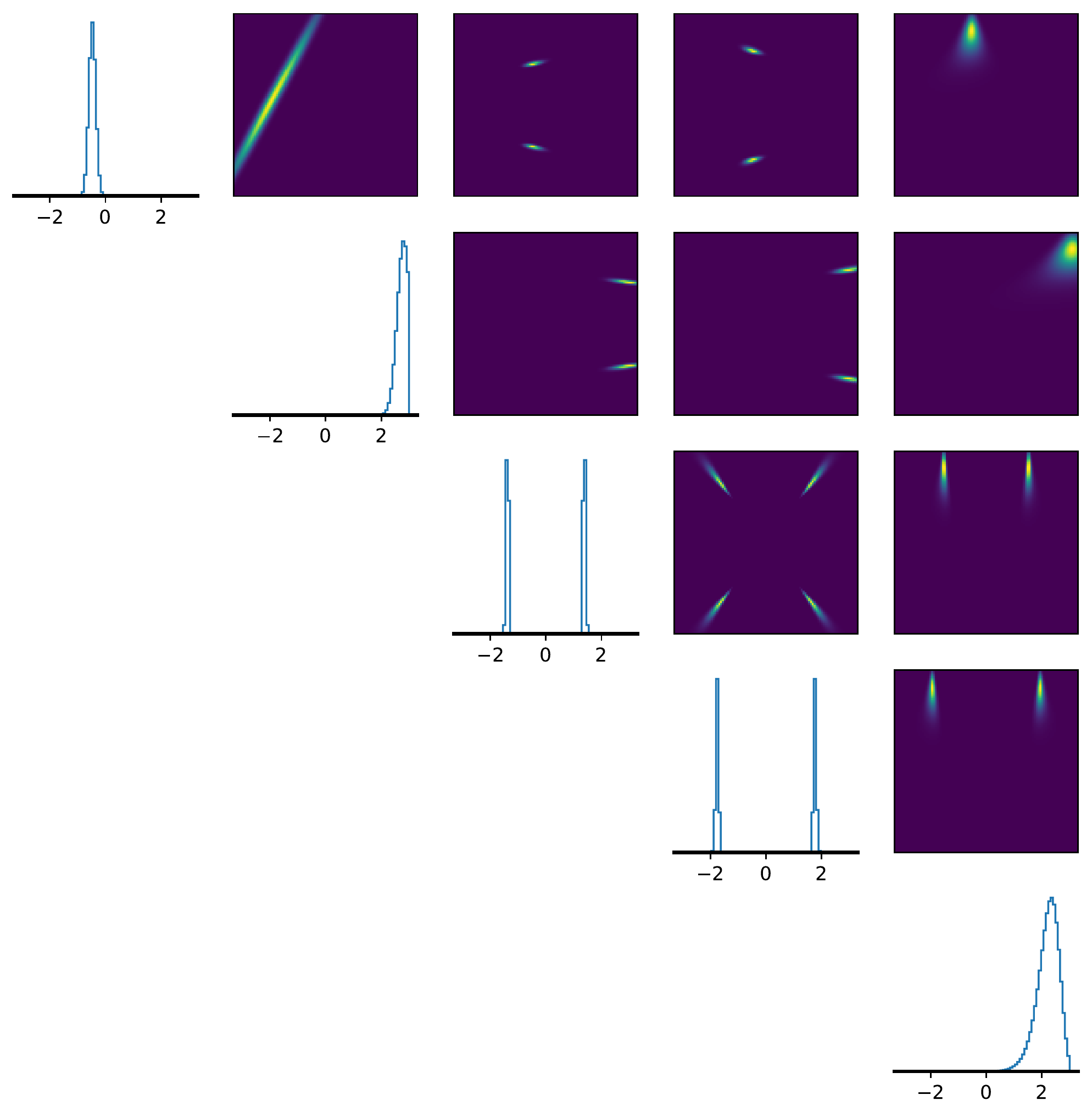}}
    \caption{Left: conditional pairplots $q_\psi(\theta_i, \theta_j|x_o, \theta_{\neg i,j})$ of SUNLE+DIVI's posterior estimate. Right: ground-truth conditional pairplots.}
\end{figure}

\subsection{Using SUNLE in a Real World neuroscience model}\label{sec:real_life}

We investigate further the performance of SUNLE by running its inference procedure on a simulator
model of a pyloric network located in stomatogastric ganglion (STG) of the crab \emph{Cancer borealis} given an observed an neuronal recording \citep{haddad10recordings}.
This model simulates 3 neurons, whose behaviors are governed by synapses and membrane conductances that act as simulator parameters $\theta$ of dimension 31.
The simulated observations are composed of 15 summary statistics of the voltage traces produced by neurons of this network \citep{prinz2003alternative,prinz2004similar}. 
The small volume of physiologically plausible regions of the parameter space $\Theta$, coupled with the nonlinearity and high computational cost of running the model, make it  a particular challenge for computational neuroscientists to fit to data (i.e., to characterize the regions of high probability of the posterior on $\theta$). Indeed,  fewer than
 1\% of draws from the prior on $\theta$ result in neural traces with well-defined summary statistics. Amortized SBI methods  require tens of millions of samples for this problem;   currently, the most sample-efficient targeted inference method is a variant of SNLE called SNVI \citep{glockler2021variational} which uses 30 rounds, each simulating  10000 samples.
\begin{figure}[htbp]
    \vspace{-2em}
    \centering
    \includegraphics[width=\columnwidth]{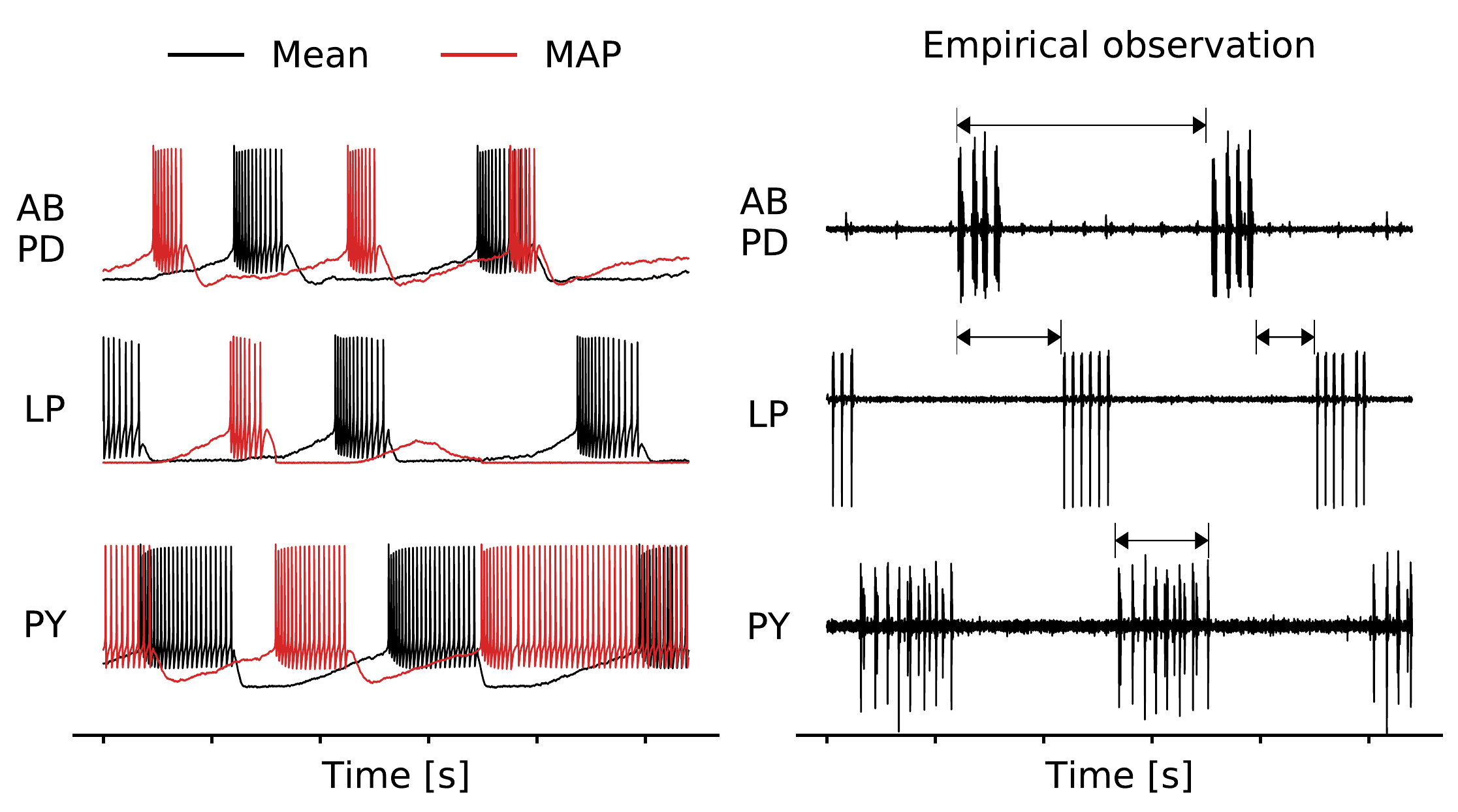}
    \includegraphics[width=\columnwidth]{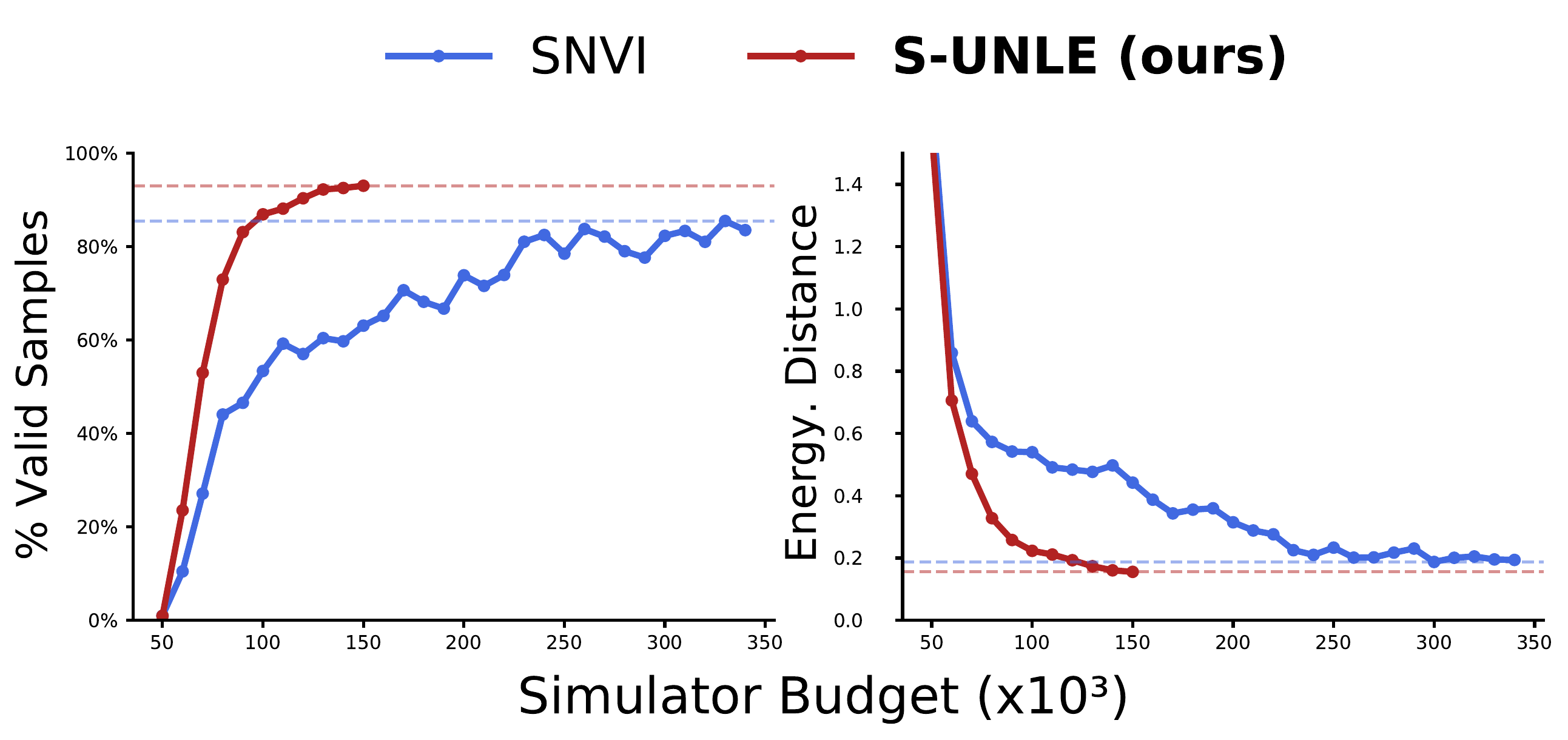}
    \caption{Inference with SUNLE on a model of the pyloric network.
    Top-left: simulations obtained by using the final posterior mean and maximum a posteriori (MAP) as a parameter.
    Top-right: the empirical observation $x_o$: arrows indicate the summary statistics. 
    Bottom-left: fraction of simulated observations with well-defined summary statistics (higher is better) at each round for SNVI and SUNLE,
    with dashed lines indicating the maximum fraction for each method. Bottom-right: performance of the posterior using the Energy Distance. 
	}
	\label{fig:pyloric}
    \vspace{-2.5em}
\end{figure}
We perform targeted inference on this model using SUNLE with a MLP of 9 layers and 300 hidden units per layers for the energy $E_{\psi}$. To maximize performance, and keeping in mind the high dimensionality of $\theta$, we use doubly intractable MCMC instead of DIVI to draw new proposal parameters across rounds. All inference and training steps are initialized using the previously available MCMC chains and EBM parameters.
We report in \cref{fig:pyloric} the evolution of the rate of simulated obvservations with valid summary statistics,  - a metric indicative of posterior quality - as well as the Energy-Scoring Rule \citep{gneiting2007strictly} of SUNLE and SNVI's posteriors across rounds. The synthetic observation simulated using SUNLE's posterior mean closely matches the empirical observation (\cref{fig:pyloric}, Left vs Center). 
As shown in \cref{fig:pyloric}, SUNLE matches the performance of SNVI in only 5 rounds, reducing by 6 the simulation budget of SNVI to achieve a comparable inference quality. After 10 rounds, SUNLE's poterior significantly exceeds the performance of SNVI in terms of number of valid samples obtained by taking the final posterior samples as parameters. The total procedure takes only 3 hours (half of which is spent simulating samples), \emph{10 times less than SNVI}.


\vspace{-1.1em}

\paragraph{Conclusion} The expanding range of applications of SBI poses new challenges
to the way SBI algorithms model data. In this work, we presented SBI methods that use an expressive Energy-Based Model as their inference engine, fitted using Maximum Likelihood. We demonstrated promising performance on synthetic benchmarks and on a real-world neuroscience model. In future work, we hope to see applications of this method to other fields where EBMs have been proven successful, such as physics \citep{noe2019boltzmann} or protein modelling \citep{ingraham2018learning}.

\pagebreak


\bibliography{biblio}

\begin{thebibliography}{53}
\providecommand{\natexlab}[1]{#1}
\providecommand{\url}[1]{\texttt{#1}}
\expandafter\ifx\csname urlstyle\endcsname\relax
  \providecommand{\doi}[1]{doi: #1}\else
  \providecommand{\doi}{doi: \begingroup \urlstyle{rm}\Url}\fi

\bibitem[Alquier et~al.(2016)Alquier, Friel, Everitt, and
  Boland]{alquier2016noisy}
Alquier, P., Friel, N., Everitt, R., and Boland, A.
\newblock Noisy monte carlo: Convergence of markov chains with approximate
  transition kernels.
\newblock \emph{Statistics and Computing}, 2016.

\bibitem[Alsing et~al.(2018)Alsing, Wandelt, and Feeney]{Alsing:2018}
Alsing, J., Wandelt, B., and Feeney, S.
\newblock Massive optimal data compression and density estimation for scalable,
  likelihood-free inference in cosmology.
\newblock \emph{Monthly Notices of the Royal Astronomical Society}, 2018.

\bibitem[Andrieu et~al.(2003)Andrieu, De~Freitas, Doucet, and
  Jordan]{andrieu2003introduction}
Andrieu, C., De~Freitas, N., Doucet, A., and Jordan, M.~I.
\newblock An introduction to {MCMC} for machine learning.
\newblock \emph{Machine learning}, 2003.

\bibitem[Beaumont et~al.(2002)Beaumont, Zhang, and
  Balding]{beaumont2002approximate}
Beaumont, M.~A., Zhang, W., and Balding, D.~J.
\newblock Approximate {B}ayesian computation in population genetics.
\newblock \emph{Genetics}, 2002.

\bibitem[Bickel \& Doksum(2015)Bickel and Doksum]{bickel2015mathematical}
Bickel, P.~J. and Doksum, K.~A.
\newblock \emph{Mathematical Statistics: Basic Ideas and Selected Topics,
  volumes I-II package}.
\newblock Chapman and Hall/CRC, 2015.

\bibitem[Brown(1986)]{brown1986fundamentals}
Brown, L.~D.
\newblock Fundamentals of statistical exponential families: with applications
  in statistical decision theory.
\newblock Ims, 1986.

\bibitem[Chopin et~al.(2020)Chopin, Papaspiliopoulos,
  et~al.]{chopin2020introduction}
Chopin, N., Papaspiliopoulos, O., et~al.
\newblock \emph{An introduction to sequential Monte Carlo}.
\newblock Springer, 2020.

\bibitem[Cornish et~al.(2020)Cornish, Caterini, Deligiannidis, and
  Doucet]{Cornish:2020}
Cornish, R., Caterini, A.~L., Deligiannidis, G., and Doucet, A.
\newblock Relaxing bijectivity constraints with continuously indexed
  normalising flows.
\newblock In \emph{International Conference on Machine Learning}, 2020.

\bibitem[Cranmer et~al.(2020)Cranmer, Brehmer, and
  Louppe]{cranmer_frontier_2020}
Cranmer, K., Brehmer, J., and Louppe, G.
\newblock The frontier of simulation-based inference.
\newblock \emph{Proceedings of the National Academy of Sciences}, 117\penalty0
  (48):\penalty0 30055--30062, 2020.

\bibitem[Deistler et~al.(2021)Deistler, Macke, and Gon{\c
  c}alves]{Deistler2021dis}
Deistler, M., Macke, J.~H., and Gon{\c c}alves, P.~J.
\newblock Disparate energy consumption despite similar network activity.
\newblock \emph{bioRxiv}, 2021.

\bibitem[Del~Moral et~al.(2006)Del~Moral, Doucet, and Jasra]{Del-Moral:2006}
Del~Moral, P., Doucet, A., and Jasra, A.
\newblock Sequential {M}onte {C}arlo samplers.
\newblock \emph{Journal of the Royal Statistical Society: Series \textup{B}},
  2006.

\bibitem[Du \& Mordatch(2019)Du and Mordatch]{du_implicit_2020}
Du, Y. and Mordatch, I.
\newblock Implicit generation and modeling with energy based models.
\newblock \emph{Advances in Neural Information Processing Systems}, 32, 2019.

\bibitem[Eberl(2003)]{Eberl:2003}
Eberl, T.
\newblock Nuclear instruments and methods in physics research, section a:
  Accelerators, spectrometers, detectors and associated.
\newblock \emph{Nucl. Instrum. Methods Phys. Res., A}, 2003.

\bibitem[Everitt(2012)]{everitt2012bayesian}
Everitt, R.~G.
\newblock Bayesian parameter estimation for latent markov random fields and
  social networks.
\newblock \emph{Journal of Computational and graphical Statistics}, 21\penalty0
  (4):\penalty0 940--960, 2012.

\bibitem[Frostig et~al.(2018)Frostig, Johnson, and Leary]{frostig2018compiling}
Frostig, R., Johnson, M.~J., and Leary, C.
\newblock Compiling machine learning programs via high-level tracing.
\newblock \emph{Systems for Machine Learning}, 2018.

\bibitem[Ghahramani(2015)]{ghahramani2015probabilistic}
Ghahramani, Z.
\newblock Probabilistic machine learning and artificial intelligence.
\newblock \emph{Nature}, 2015.

\bibitem[Gl{\"o}ckler et~al.(2021)Gl{\"o}ckler, Deistler, and
  Macke]{glockler2021variational}
Gl{\"o}ckler, M., Deistler, M., and Macke, J.~H.
\newblock Variational methods for simulation-based inference.
\newblock In \emph{International Conference on Learning Representations}, 2021.

\bibitem[Gneiting \& Raftery(2007)Gneiting and Raftery]{gneiting2007strictly}
Gneiting, T. and Raftery, A.~E.
\newblock Strictly proper scoring rules, prediction, and estimation.
\newblock \emph{Journal of the American statistical Association}, 2007.

\bibitem[Grathwohl et~al.(2020)Grathwohl, Wang, Jacobsen, Duvenaud, Norouzi,
  and Swersky]{grathwohl2020your}
Grathwohl, W., Wang, K.-C., Jacobsen, J.-H., Duvenaud, D., Norouzi, M., and
  Swersky, K.
\newblock Your classifier is secretly an energy based model and you should
  treat it like one.
\newblock \emph{arXiv preprint arXiv:1912.03263}, 2020.

\bibitem[Greenberg et~al.(2019)Greenberg, Nonnenmacher, and
  Macke]{Greenberg:2019}
Greenberg, D., Nonnenmacher, M., and Macke, J.
\newblock Automatic posterior transformation for likelihood-free inference.
\newblock In \emph{International Conference on Machine Learning}, 2019.

\bibitem[Haddad \& Marder(2021)Haddad and Marder]{haddad10recordings}
Haddad, S.~A. and Marder, E.
\newblock Recordings from the c. borealis stomatogastric nervous system at
  different temperatures in the decentralized condition.
\newblock \emph{URL https://doi. org/10.5281/zenodo}, July 2021.

\bibitem[Hastie et~al.(2009)Hastie, Friedman, and
  Tisbshirani]{hastie_friedman_tisbshirani_2009}
Hastie, T., Friedman, J., and Tisbshirani, R.
\newblock \emph{2.4 Statistical Decision Theory}, pp.\ ~18.
\newblock Springer, 2 edition, 2009.

\bibitem[Hermans et~al.(2020)Hermans, Begy, and Louppe]{Hermans:2020}
Hermans, J., Begy, V., and Louppe, G.
\newblock Likelihood-free {MCMC} with amortized approximate ratio estimators.
\newblock In \emph{International Conference on Machine Learning}, 2020.

\bibitem[Hyv{\"a}rinen \& Dayan(2005)Hyv{\"a}rinen and
  Dayan]{hyvarinen2005estimation}
Hyv{\"a}rinen, A. and Dayan, P.
\newblock Estimation of non-normalized statistical models by score matching.
\newblock \emph{Journal of Machine Learning Research}, 2005.

\bibitem[Ingraham et~al.(2018)Ingraham, Riesselman, Sander, and
  Marks]{ingraham2018learning}
Ingraham, J., Riesselman, A., Sander, C., and Marks, D.
\newblock Learning protein structure with a differentiable simulator.
\newblock In \emph{International Conference on Learning Representations}, 2018.

\bibitem[Kelly \& Grathwohl(2021)Kelly and Grathwohl]{kelly2021no}
Kelly, J. and Grathwohl, W.~S.
\newblock No conditional models for me: Training joint ebms on mixed continuous
  and discrete data.
\newblock In \emph{Energy Based Models Workshop-ICLR 2021}, 2021.

\bibitem[Khemakhem et~al.(2020)Khemakhem, Monti, Kingma, and
  Hyvarinen]{khemakhem2020ice}
Khemakhem, I., Monti, R., Kingma, D., and Hyvarinen, A.
\newblock Ice-beem: Identifiable conditional energy-based deep models based on
  nonlinear ica.
\newblock \emph{Advances in Neural Information Processing Systems}, 2020.

\bibitem[Kong \& Chaudhuri(2020)Kong and Chaudhuri]{kong20expressive}
Kong, Z. and Chaudhuri, K.
\newblock The expressive power of a class of normalizing flow models.
\newblock In \emph{Proceedings of the Twenty Third International Conference on
  Artificial Intelligence and Statistics}, 26--28 Aug 2020.

\bibitem[LeCun et~al.(2006)LeCun, Chopra, Hadsell, Ranzato, and
  Huang]{LeCun:2006}
LeCun, Y., Chopra, S., Hadsell, R., Ranzato, M., and Huang, F.
\newblock A tutorial on energy-based learning.
\newblock \emph{Predicting structured data}, 2006.

\bibitem[Lotka(1920)]{lotka1920analytical}
Lotka, A.~J.
\newblock Analytical note on certain rhythmic relations in organic systems.
\newblock \emph{Proceedings of the National Academy of Sciences}, 1920.

\bibitem[Lueckmann et~al.(2021)Lueckmann, Boelts, Greenberg, Gonçalves, and
  Macke]{lueckmann2021benchmarking}
Lueckmann, J.-M., Boelts, J., Greenberg, D.~S., Gonçalves, P.~J., and Macke,
  J.~H.
\newblock Benchmarking simulation-based inference.
\newblock In \emph{Proceedings of the 24th International Conference on
  Artificial Intelligence and Statistics (AISTATS)}, 2021.

\bibitem[Marin et~al.(2012)Marin, Pudlo, Robert, and Ryder]{Marin:2012}
Marin, J.-M., Pudlo, P., Robert, C.~P., and Ryder, R.~J.
\newblock Approximate {B}ayesian computational methods.
\newblock \emph{Statistics and Computing}, 2012.

\bibitem[Markram et~al.(2015)Markram, Muller, Ramaswamy, Reimann, Abdellah,
  Sanchez, Ailamaki, Alonso-Nanclares, Antille, Arsever, et~al.]{Markram:2015}
Markram, H., Muller, E., Ramaswamy, S., Reimann, M.~W., Abdellah, M., Sanchez,
  C.~A., Ailamaki, A., Alonso-Nanclares, L., Antille, N., Arsever, S., et~al.
\newblock Reconstruction and simulation of neocortical microcircuitry.
\newblock \emph{Cell}, 2015.

\bibitem[M{\o}ller et~al.(2006)M{\o}ller, Pettitt, Reeves, and
  Berthelsen]{moller2006efficient}
M{\o}ller, J., Pettitt, A.~N., Reeves, R., and Berthelsen, K.~K.
\newblock An efficient markov chain monte carlo method for distributions with
  intractable normalising constants.
\newblock \emph{Biometrika}, 2006.

\bibitem[Murray et~al.(2006)Murray, Ghahramani, and MacKay]{murray2006doubly}
Murray, I., Ghahramani, Z., and MacKay, D. J.~C.
\newblock Mcmc for doubly-intractable distributions.
\newblock In \emph{Proceedings of the Twenty-Second Conference on Uncertainty
  in Artificial Intelligence}, 2006.

\bibitem[Nijkamp et~al.(2019)Nijkamp, Hill, Zhu, and Wu]{nijkamp2019learning}
Nijkamp, E., Hill, M., Zhu, S.-C., and Wu, Y.~N.
\newblock Learning non-convergent non-persistent short-run mcmc toward
  energy-based model.
\newblock \emph{Advances in Neural Information Processing Systems}, 2019.

\bibitem[No{\'e} et~al.(2019)No{\'e}, Olsson, K{\"o}hler, and
  Wu]{noe2019boltzmann}
No{\'e}, F., Olsson, S., K{\"o}hler, J., and Wu, H.
\newblock Boltzmann generators: Sampling equilibrium states of many-body
  systems with deep learning.
\newblock \emph{Science}, 2019.

\bibitem[Pacchiardi \& Dutta(2022)Pacchiardi and Dutta]{pacchiardi2020score}
Pacchiardi, L. and Dutta, R.
\newblock Score matched neural exponential families for likelihood-free
  inference.
\newblock \emph{Journal of Machine Learning Research}, 2022.

\bibitem[Papamakarios et~al.(2019)Papamakarios, Sterratt, and
  Murray]{papamakarios_2019_sequential}
Papamakarios, G., Sterratt, D., and Murray, I.
\newblock Sequential neural likelihood: Fast likelihood-free inference with
  autoregressive flows.
\newblock In \emph{The 22nd International Conference on Artificial Intelligence
  and Statistics}, 2019.

\bibitem[Pospischil et~al.(2008)Pospischil, Toledo-Rodriguez, Monier,
  Piwkowska, Bal, Fr{\'e}gnac, Markram, and Destexhe]{Pospischil:2008}
Pospischil, M., Toledo-Rodriguez, M., Monier, C., Piwkowska, Z., Bal, T.,
  Fr{\'e}gnac, Y., Markram, H., and Destexhe, A.
\newblock Minimal {H}odgkin--{H}uxley type models for different classes of
  cortical and thalamic neurons.
\newblock \emph{Biological Cybernetics}, 2008.

\bibitem[Prinz et~al.(2003)Prinz, Billimoria, and Marder]{prinz2003alternative}
Prinz, A.~A., Billimoria, C.~P., and Marder, E.
\newblock Alternative to hand-tuning conductance-based models: construction and
  analysis of databases of model neurons.
\newblock \emph{Journal of neurophysiology}, 2003.

\bibitem[Prinz et~al.(2004)Prinz, Bucher, and Marder]{prinz2004similar}
Prinz, A.~A., Bucher, D., and Marder, E.
\newblock Similar network activity from disparate circuit parameters.
\newblock \emph{Nature neuroscience}, 2004.

\bibitem[Raginsky et~al.(2017)Raginsky, Rakhlin, and
  Telgarsky]{raginsky2017non}
Raginsky, M., Rakhlin, A., and Telgarsky, M.
\newblock Non-convex learning via stochastic gradient langevin dynamics: a
  nonasymptotic analysis.
\newblock In \emph{Conference on Learning Theory}, 2017.

\bibitem[Schafer \& Freeman(2012)Schafer and Freeman]{Schafer:2012}
Schafer, C.~M. and Freeman, P.~E.
\newblock Likelihood-free inference in cosmology: Potential for the estimation
  of luminosity functions.
\newblock In \emph{Statistical Challenges in Modern Astronomy V}. Springer,
  2012.

\bibitem[Sj{\"o}strand et~al.(2008)Sj{\"o}strand, Mrenna, and
  Skands]{Sjostrand:2008}
Sj{\"o}strand, T., Mrenna, S., and Skands, P.
\newblock A brief introduction to pythia 8.1.
\newblock \emph{Computer Physics Communications}, 2008.

\bibitem[Song \& Kingma(2021)Song and Kingma]{song2021train}
Song, Y. and Kingma, D.~P.
\newblock How to train your energy-based models.
\newblock \emph{arXiv preprint arXiv:2101.03288}, 2021.

\bibitem[Song et~al.(2020)Song, Garg, Shi, and Ermon]{song2020sliced}
Song, Y., Garg, S., Shi, J., and Ermon, S.
\newblock Sliced score matching: A scalable approach to density and score
  estimation.
\newblock In \emph{Uncertainty in Artificial Intelligence}, 2020.

\bibitem[Tieleman(2008)]{tieleman2008training}
Tieleman, T.
\newblock Training restricted boltzmann machines using approximations to the
  likelihood gradient.
\newblock In \emph{Proceedings of the 25th international conference on Machine
  learning}, 2008.

\bibitem[Wainwright \& Jordan(2008)Wainwright and Jordan]{Wainwright:2008}
Wainwright, M.~J. and Jordan, M.~I.
\newblock \emph{Graphical models, exponential families, and variational
  inference}.
\newblock Now Publishers Inc, 2008.

\bibitem[Wenliang \& Kanagawa(2020)Wenliang and
  Kanagawa]{wenliang2020blindness}
Wenliang, L.~K. and Kanagawa, H.
\newblock Blindness of score-based methods to isolated components and mixing
  proportions.
\newblock \emph{arXiv preprint arXiv:2008.10087}, 2020.

\bibitem[Wood(2010)]{wood2010statistical}
Wood, S.~N.
\newblock Statistical inference for noisy nonlinear ecological dynamic systems.
\newblock \emph{Nature}, 2010.

\bibitem[Xie et~al.(2021)Xie, Zhu, Li, and Li]{xie2021tale}
Xie, J., Zhu, Y., Li, J., and Li, P.
\newblock A tale of two flows: Cooperative learning of langevin flow and
  normalizing flow toward energy-based model.
\newblock In \emph{International Conference on Learning Representations}, 2021.

\bibitem[Zhang et~al.(2022)Zhang, Key, Hayes, Barber, Paige, and
  Briol]{zhang2022towards}
Zhang, M., Key, O., Hayes, P., Barber, D., Paige, B., and Briol, F.-X.
\newblock Towards healing the blindness of score matching.
\newblock \emph{arXiv preprint arXiv:2209.07396}, 2022.

\end{thebibliography}
\bibliographystyle{icml2023}

\newpage
\appendix
\onecolumn

\section*{Supplementary Material for the paper \emph{Maximum Likelihood Learning of Energy-Based Models for Simulation-Based Inference}}

The supplementary materials include the following:

\textbf{Appendix A:}
\begin{itemize}
	\item A discussion in \cref{app-subsec:ebm-doubly-intractable} of the computational rationale motivating the tilting approach of AUNLE.
	\item An EBM training method in \cref{app-subsec:algorithms} which uses the family of Sequential
		Monte Carlo (SMC) samplers to efficiently approximate expectations under the EBM
		during approximate likelihood maximization. We show that using these new methods can lead to increased stability and performance for a fixed budget.
\end{itemize}

\textbf{Appendix B:}
\begin{itemize}
    \item A conditional EBM training method for SUNLE in \cref{app-subsec:sunle-likelihood}.
    \item A proof of \cref{prop: cond-exp-l2-opt} in \cref{app-subsec:lemma-proof}.
    \item Empirical improvements to DIVI in \cref{app-subsec:emp-improvements-divi}.
    \item A method for training $\lz_\eta$ online in \cref{app-subsec:online-lznet-training}.
\end{itemize}

\textbf{Appendix C:}
\begin{itemize}
    \item Figures in \cref{app-subsec:posteriors} of UNLE's posterior samples for SBI benchmark problems.
    \item A discussion in \cref{app-subsec:short-run} about the (absence of) the short-run effect \citep{nijkamp2019learning} in UNLE.
	\item An experiment in \cref{app-sec:tilting-validation} that suggests that the $(Z,\theta)$-uniformization of AUNLE's posterior holds in practice in learned AUNLE models.
	\item A detailed computational analysis in \cref{app-sec:computational-cost} of AUNLE and SUNLE, which prove highly competitive over alternatives.
    \item Details of the experimental setups for SNLE and SMNLE in \cref{app-subsec:experimental-setup}.
    \item Finally, we provide additional details in \cref{app-sec:neuroscience} on the results of SUNLE on the pyloric network: we provide an estimation of the pairwise marginals of the final posterior, which contains patterns also present in the pairwise marginals obtained by \cite{glockler2021variational}.
\end{itemize}

\newpage
\section{AUNLE: Methodological Details}
\subsection{Energy-Based Models as Doubly-Intractable Joint Energy-Based Models} \label{app-subsec:ebm-doubly-intractable}
AUNLE learns a likelihood model $ q_{\psi}(x|\theta) $ by minimizing the
likelihood of a tillted joint EBM $ \frac{ p(\theta) e^{-E_{\psi}(x, \theta)}
}{ Z_{\pi}(\psi) }$. While the gain in tractability arising in AUNLE's posterior suffices to motivate the use of this model, another computational argument holds. Consider the non-tilted joint model:
\begin{equation*}
\begin{aligned}
	\pi(\theta) \frac{ q_{\psi}(x|\theta) }{ Z(\theta, \psi) }.
\end{aligned}
\end{equation*}
Expectations under this model can be computed by running a MCMC chain implementing a Metropolis-Within-Gibbs sampling method as in \cite{kelly2021no}, which uses:
\begin{itemize}
	\item  any proposal distribution for $ q_{\pi, \psi}(x|\theta) \propto q_{\psi}(x|\theta) $, such as a MALA proposal;
	\item  an approximate doubly-intractable MCMC kernel step for $ q_{\pi, \psi}(\theta | x) \propto \pi(\theta) \frac{ e^{-E_{\psi}(x, \theta)} }{ Z_{\pi}(\theta) } $ which  is doubly-intractable.
\end{itemize}

However, running the approximate doubly-intractable MCMC kernel step requires
sampling from $ q_{\psi}(x|\theta) $, incurring an additional nested loop during
training. Thus, naive MCMC-based Maximum-Likelihood optimization of untilted joint EBM is prohibitive from a computational point of view.

\subsection{Training EBMs using Sequential Monte Carlo}\label{app-subsec:algorithms}

The EBM training procedure referenced in \cref{alg:a-unle} is as follows:

\begin{algorithm}[H] \caption{$\texttt{maximize\_ebm\_log\_l}(\mathcal D, \psi_0)$}
\label{alg:ebm-ml-training}
	\vspace{0.3em}
	\begin{algorithmic}
		\STATE \hspace{-1em}\textbf{Input:} Training Data $ \mathcal D:=\{ x^i, \theta^i \}_{i=1}^{N} $, Initial EBM parameters $ \psi_0 $
		\vspace{0.20em}
		\STATE \hspace{-1em}{\bf Output:} Density estimator $ q_{\psi}(x, \theta) $
		\vspace{0.20em}
		\STATE \hspace{-1em}{\bf Initialize} \hspace{-0.em}$q_{\psi_0}(x) \hspace{-0.1em}\propto e^{-E_{\psi_0}(x, \theta)}, \hat{q}_0 \propto \sum_i \delta_{(x^i, \theta^i)}$ 
		\vspace{-0.2em}
		\STATE \hspace{-1em}{{\bf for} $k=0,\dots,K-1$ {\bf do}}
		\vspace{0.20em}
		\STATE \hspace{-1em}\hspace{1em} $ \widehat{ q }:= \verb+make_particle_approx+(q_{\psi_k}, \widehat{ q })$
		\vspace{0.20em}
		\STATE \hspace{-1em}\hspace{1em} $ \widehat{ G } = -\frac{1}{N}\sum \nabla_{ \psi }  E_{\psi_k}(x^i, \theta^i) +\hspace{-0.15em}\mathbb{E}_{ \widehat{ q } } \nabla_{ \psi }  E_{\psi_k}(x, \theta)$
		\STATE \hspace{-1em}\hspace{1em} $ \psi_{k+1} = \verb+ADAM+(\psi_k, \widehat{ G })$
		\vspace{0.20em}
		\STATE \hspace{-1em}{{\bf end for}}
		\vspace{0.20em}
		\STATE \hspace{-1em}{\bf Return} $ q_{\psi_{K}}$
	\end{algorithmic}
\end{algorithm}

The routine $\texttt{make\_particle\_approx}(q, \hat{q}_0)$ is a generic routine that produces a particle approximation of a target unnormalized density $q$ with an initial particle approximation $\hat{q}_0$.

The main technique to compute particle approximations when training EBMs using \cref{alg:ebm-ml-training} is to run $N$ MCMC chains in parallel targeting the EBM 
\cite{song2021train};
aggregating the final samples $y_i$ of each chain $i$ yields a particle approximation
$q = \frac{1}{N}\sum_i \delta y_i$ of the EBM in question.
In this section, we describe an alternative \texttt{make\_ebm\_approx} which  \emph{efficiently} constructs EBM particle approximations across iterations of \cref{alg:ebm-ml-training} through a Sequential Monte Carlo (SMC) algorithm \citep{chopin2020introduction,Del-Moral:2006}.
In addition to its efficienty, this new routine does not suffer from the bias of incurred by the use of finitely many steps in MCMC-based methods.
We apply this routine within the EBM training step of AUNLE's, and show that the learned posteriors
can be more accurate than MCMC methods for a fixed compute power allocated to training.

\subsubsection{Background: Sequential Monte Carlo Samplers}
Sequential Monte Carlo (SMC) Samplers \citep{chopin2020introduction, Del-Moral:2006} are a family of efficient Importance Sampling (IS)-based algorithms,
that address the same problem as the one of MCMC, namely computing a normalized particle
approximation of a target density $q$  known up to a normalizing constant $ Z$. The particle approximation $ \widehat{
q }_{SMC} $ computed by SMC samplers (consisting of $N$ \emph{particles} $
y^{i} $, like in MCMC methods, but weighted non-uniformly by some weights
$w^{i}$) is produced by defining a set of $L$ intermediate densities
$(\nu_{l})_{l=0}^L$ bridging between the target density $\nu_l {=} q$ and some
initial density $\nu_0,$ for which a particle approximation $
\nu^{N}_0:\sum_i w_0^i \delta_{y_0^i}$ is readily available. The intermediate
densities are often chosen to be a geometric interpolation between $\nu_0$ and
$\nu_L$, i.e. $\nu_l \propto (\nu_0)^{1-\frac{l}{l}}(\nu_L)^{\frac{l}{l}}$, so
that $\nu_l$ are also known up to some normalizing constant.
SMC samplers sequentially constructs an approximation
$\nu^{N}_l := \sum w^i_l \delta_{y^i_l} $ to the respective density $\nu_l$  at time $l$,
using previously computed approximations of $\nu_{l-1}$ at time $l-1$.  At each time step, the approximations are obtained by applying
three successive operations: \emph{Importance Sampling}, \emph{Resampling} and
\emph{MCMC} sampling.
We provide a vanilla SMC sampler implementation in \cref{alg:smc}, and refer to this algorithm as \texttt{SMC}.

\begin{algorithm}
\caption{\texttt{SMC}$(q,\nu_0,\nu_0^N)$}\label{alg:smc}
	\begin{algorithmic}[1]
		\STATE \textbf{Hyper-parameters:} Number of particles $N$,  number of steps $L$, 
		re-sampling threshold $A\in [\frac{1}{N},1)$.
		\STATE \textbf{Input:} Target density $q$, initial density $\nu_0$, particle approximations $\nu_0^N$ and $\nu_0$
		\STATE {\bf Output:} Particle approximations to $q$.
		\STATE Construct geometric path $\left(\nu_l\right)_{l=1}^L$ from $\nu_0$ and $q$.
			\FOR{$l=1,\dots,L$}
			\STATE Compute IS weights $w_l^i$ and $W_l^i$
			\STATE Draw $N$ samples $(\widetilde{Y}_{l}^i)_{i=1}^N$ from $(Y^i_{l-1})_{i=1}^N$ according to weights $(W_l^i)_{i=1}^N$, then set $W_l^i {=} \frac{1}{N}$. 
			\STATE  Sample $Y_{l}^i \sim \mathcal{K}_{l}(\widetilde{Y}_{l}^i,\cdot)$ using Markov kernel $\mathcal{K}_l$.
			\ENDFOR
			\STATE Return approximation $q^N_{SMC} {:=} \left(Y_L^i,W_L^i\right)_{i=1}^N$.
	\end{algorithmic}
\end{algorithm}

Importantly, under mild assumptions, the particle approximation constructed by SMC
provides consistent estimates of expectations of any function $f$ under the
target $q$:
\begin{align*}
	\sum_{i=1}^N w^{i}f(y^{i}) \xrightarrow{P} \mathbb{E}_{y\sim q}\left[f(y)\right].
\end{align*}
We briefly compare the role played by the number of steps and particles in both MCMC and SMC algorithms:

{\bfseries Number of particles}
SMC samplers differ from MCMC samplers in their origin of their bias:
while the bias of MCMC methods comes from running the chain for a finite number
of steps only, the bias of SMC methods comes from the use of finitely many particles.

{\bfseries Number of steps}
While it is usually beneficial to use a high number of iterations within MCMC samplers to decrease algorithm bias and ensure that the stationnary distribution is reached, the number of steps (or intermediate distributions) in SMC is beneficial to ensure a smooth transition from the proposal to the target distribution: however, the variance of SMC samplers as a function of the number of steps is not guaranteed to be decreasing  even if variance bounds that are uniform in the number of steps can be derived by making assumptions on $\mathcal K_l$ \cite{chopin2020introduction}. When applying \verb+SMC+ within AUNLE's training loop, we find that using more SMC samplers steps usually increase the quality of the final posterior.

In \cref{app-subsec:smc-aunle}, we describe how to use \texttt{SMC} routine efficiently to
approximate EBM expectations within \cref{alg:ebm-ml-training}.

\subsubsection{Efficient use of SMC during AUNLE training using OG-SMC} \label{app-subsec:smc-aunle}

A naive approach which uses the \texttt{SMC} routine of \cref{alg:smc} within the EBM training loop of \cref{alg:ebm-ml-training} would consist in calling the \texttt{SMC} at every training iteration using a fixed, predefined proposal density $\nu_0$ and associated particle approximation and $\hat{\nu}_0$, such as one from a standard gaussian distribution. However, as training goes, the EBM is likely to differ significantly from the proposal density $q_0$, requiring the use of many SMC inner steps to obtain a good particle approximation.

\emph{A more efficient approach}, which we propose, is to use the readily available particle unnormalized EBM density $q_{\psi^{k-1}}$ and associated particle approximation $\hat{q}^k$ computed by \texttt{SMC} at the iteration k-1 \textbf{as the input} to the call to \texttt{SMC} targeting the EBM $q_{\psi^k}$ at iteration k. \cref{alg:aunle-og-smc} implements this approach.

\begin{algorithm}[H] \caption{SMC-powered ML training of EBMs}\label{alg:aunle-og-smc}
	\vspace{0.3em}
	\begin{algorithmic}
		\STATE \hspace{-1em}\textbf{Input:} Training Data $ \{ x^{(i)} \}_{i=1}^{N} $, Initial EBM parameters $ \psi_0 $
		\vspace{0.20em}
		\STATE \hspace{-1em}{\bf Output:} Density estimator $ q_{\psi}(x) $
		\vspace{0.20em}
		\STATE \hspace{-1em}{\bf Initialize} $q_{\psi_0}(x) \propto e^{-E_{\psi_0}(x)}, q_{-1}=\nu_0, \hat{q}_{-1}=\hat{\nu}_0$
		\vspace{0.20em}
		\STATE \hspace{-1em}{{\bf for} $i=0,\dots,\verb+max_iter+-1$ {\bf do}}
		\vspace{0.20em}
		\color{gray}
		\STATE \hspace{-1em}\hspace{1em} $  \#\,\, \widehat{ q }:= \verb+make_particle_approx+(q_{\psi_k}, \widehat{ q })$
		\color{black}
		\STATE \hspace{-1em}\hspace{1em} $ \widehat{ q }_k:= \verb+SMC+(q_{\psi_k}, q_{k-1}, \hat{q}_{k-1})$
		\STATE \hspace{-1em}\hspace{1em} $ q_k:= q_{\psi_k}$
		\vspace{0.20em}
		\STATE \hspace{-1em}\hspace{1em} $ \widehat{ G } = - \frac{\gamma}{N}\sum \nabla_{ \psi }  E_{\psi_k}(x^i) +\hspace{-0.15em}\mathbb{E}_{ \widehat{ q } } \nabla_{ \psi }  E(x)$ \STATE \hspace{-1em}\hspace{1em} $ \psi_{k+1} = \verb+ADAM+(\psi_k, \widehat{ G })$
		\vspace{0.20em}
		\STATE \hspace{-1em}{{\bf end for}}
		\vspace{0.20em}
		\STATE \hspace{-1em}Return $ q_{\psi_{K}}$
	\end{algorithmic}
\end{algorithm}
In practice, we find that using 20 SMC intermediate densities (with 3 steps of $\mathcal K_t$) in each call to \texttt{SMC} yields a similar performance as a 250-MCMC steps EBM training procedure. By considering a more constrained budget, using only 5 SMC intermediates densities outperforms a 30-steps MCMC EBM training procedure. See Figures \ref{fig:mala_vs_smc_200v20} and \ref{fig:mala_vs_smc_30v5}, respectively.
\begin{figure}[H]
	\includegraphics[width=\textwidth]{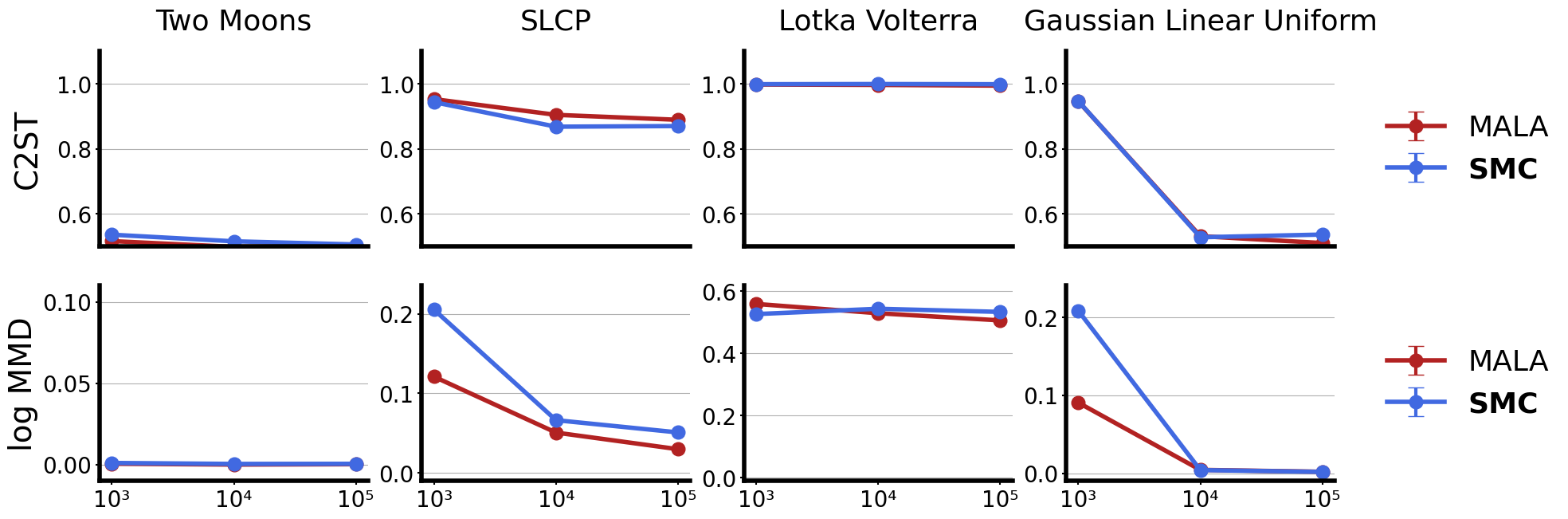}
	\caption{Performance of AUNLE, using a MCMC-powered particle approximation routine with 200 MCMC steps vs. SMC with 20 steps.}
	\label{fig:mala_vs_smc_200v20}
    \vspace{-1em}
\end{figure}
\begin{figure}[H]
	\includegraphics[width=\textwidth]{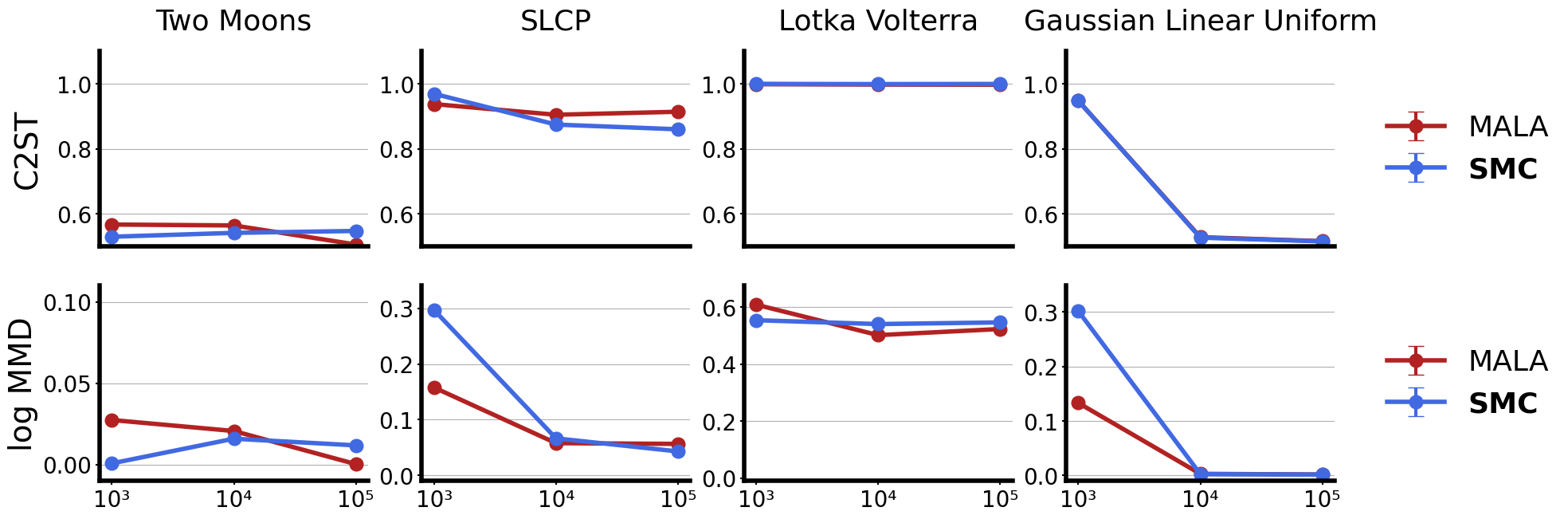}
	\caption{Performance of AUNLE, using either a MCMC-powered particle approximation routine with 30 MCMC steps vs. SMC with 5 steps.}
	\label{fig:mala_vs_smc_30v5}
    \vspace{-1em}
\end{figure}

\newpage
\section{SUNLE: Methodological Details}
\subsection{Training conditional EBMs using SMC} \label{app-subsec:sunle-likelihood}
The gradient of the conditional EBM loss \cref{eq:avg-conditional-log-l} is
\begin{equation} \label{eq:avg-conditional-log-l-grad}
    	\nabla_{ \psi } \mathcal  L_s(\psi)
	= - \frac{1}{N}\sum\limits_{ i=1 }^{ N } (\nabla_{ \psi } E_{\psi}(x^{i},
	\theta^{i}) + \overbrace{\mathbb{E}_{q_{\psi}(\cdot | \theta^{i})}
		\nabla_{ \psi }
E_{\psi}(x, \theta^{i})}\limits^{\text{intractable}})
\end{equation}
Unlike standard EBM objectives, this loss \emph{directly} targets the likelihood $q_\psi(x|\theta)$, thus bypassing the need for modeling the proposal $\pi$. We propose  \cref{alg:cebm-ml-training} a method that optimizes this objective (previously used for normalizing flows in \citealp{papamakarios_2019_sequential}). The intractable term of \cref{eq:avg-conditional-log-l-grad} is an average over the EBM probabilities conditioned on all parameters from the training set, and thus differs from the intractable term of the gradient in \eqref{eq:EBM-likelihood-and-gradient}, composed of a single integral. \cref{alg:cebm-ml-training} approximates this term during training by keeping track of one particle approximation $\widehat{q}_i = \delta_{\tilde{x}_i}$ per conditional density $q_\psi(\cdot|\theta^i)$ comprised of a single particle. The algorithm proceeds by updating only a batch of size $B$ of such particles using an MCMC update with target probability
chain $q_{\psi_k}(\cdot|\theta^i)$, where $\psi_k$ is the EBM iterate at iteration $k$ of round $r$.
Learning the likelihood using \cref{alg:cebm-ml-training} allows to use all the existing simulated data during training without re-learning the proposal, maximizing sample efficiency while minimizing learning complexity. The multi-round procedure of SUNLE is summarized in \cref{alg:s-unle}.

\begin{algorithm}[H] \caption{$\texttt{maximize\_cebm\_log\_l}(\mathcal D, \psi_0)$}\label{alg:cebm-ml-training}
	\begin{algorithmic}
		\STATE \hspace{-1em}\textbf{Input:} Training data $ \mathcal D:=\{ \theta^{(i)}, x^{(i)} \}_{i=1}^{N} $, Initial EBM parameters $ \psi_0 $
		\STATE \hspace{-1em}{\bf Output:} Cond. Density estimator $ q_{\psi}(x | \theta) $
		\STATE \hspace{-1em}{\bf Initialize} $q_{\psi_0} \propto e^{-E_{\psi_0}(\theta, x)}, \{\hat{q}_i = \delta_{x^i}\}_{i=1}^{N}$
		\STATE \hspace{-1em}{{\bf for} $k=0,\dots,K-1$ {\bf do}}
		\STATE \hspace{0.5em}\hspace{-1em}{{\bf for} $i=0,\dots,N-1$ {\bf do}}
			\STATE \hspace{-1em}\hspace{0.5em}\hspace{0.5em} $ \widehat{ q }_i:= \verb+make_particle_approx+(q_{\psi_k}(\cdot, \theta^i), \widehat{ q }_i)$
		\STATE \hspace{-1em}\hspace{0.5em}{{\bf end for}}
		\vspace{0.15em}
		\STATE \hspace{-1em}\hspace{0.5em}$ \widehat{ G }\hspace{-0.3em} = \hspace{-0.1em} -\frac{1}{N}\sum \nabla_{ \psi } \hspace{-0.1em} E_{\psi_k}(x^{i}, \theta^{i}) +\hspace{-0.15em}\mathbb{E}_{ \widehat{ q }_i } \nabla_{ \psi }  E_{\psi_k} (x^i, \theta^{i})$
		\STATE \hspace{-1em}\hspace{0.5em}$ \psi_{k+1} \hspace{-0.3em} = \verb+ADAM+(\psi_k, \widehat{ G })$
		\STATE \hspace{-1em}{{\bf end for}}
		\STATE \hspace{-1em}{\bf Return} $ q_{\psi_K}$
	\end{algorithmic}
\end{algorithm}

\subsection{Proof of \cref{prop: cond-exp-l2-opt}} \label{app-subsec:lemma-proof}
We repeat \cref{prop: cond-exp-l2-opt} below.
\begin{proposition*}
Assume that $E_\psi(\theta, x)$ is differentiable w.r.t $\theta$, and let $\mathcal F$ be the space of 1-differentiable real-valued functions on $\Theta$.
Let $\nu$ be any distribution with full support on $\Theta$, and let $f^\star \in \mathcal F$. Then $f^\star$ is a solution of: 
\begin{equation*}
    \min_{f\in\mathcal{F}} \mathbb{E}_{q_\psi(x|\theta) \nu(\theta)} \norm{\nabla f(\theta) + \nabla_\theta E_\psi(x, \theta)}^2
\end{equation*}
if and only if $f^\star = \log Z(\theta, \psi) + C$, for some constant $C$.
\end{proposition*}

\begin{proof}
The proof stems from the following definition of the conditional expectation $\mathbb{E}[Z|Y]$ for two random vectors $Z$ and $Y$ \citep{hastie_friedman_tisbshirani_2009}:
\begin{equation*}
    \mathbb{E}[Z|Y] = \argmin_{g \text{ measurable}} \mathbb{E}\left( \norm{Z-g(Y)}^2 \right)
\end{equation*}
Applying this result to $Y=\theta$ (with distribution $\nu$) and $Z=-\nabla_\theta E_\psi(x,\theta)$, where $x|\theta$ is sampled according to $q_\psi(\cdot|\theta)$, the conditional expectation $\theta \mapsto -\mathbb{E}_{q_\psi(x|\theta)} \nabla_\theta E_\psi(x,\theta)$ is thus given by
\begin{equation} \label{eq:g-problem}
    \argmin_{g \text{ measurable}} \mathbb{E}_{(x,\theta) \sim q_\psi(x|\theta)\nu(\theta)} \norm{\nabla_\theta E_\psi(x,\theta) + g(\theta)}^2
\end{equation}
As we show, the minimizers of \cref{eq:g-problem} and of \cref{prop: cond-exp-l2-opt} 
\begin{equation*}
    \min_{f\in\mathcal{F}} \mathbb{E}_{(x,\theta) \sim q_\psi(x|\theta)\nu(\theta)} \norm{\nabla_\theta E_\psi(x,\theta) + \nabla f(\theta)}^2.
\end{equation*}
are connected: Indeed, consider any primitive $f^*_C$ of the conditional expectation function $g: \theta \mapsto -\mathbb{E}_{q_\psi(x|\theta)} \nabla_\theta E_\psi(x,\theta)$. By \cref{lem:diff-log-z}, $f^*_C$ is given by $\log Z(\theta, \psi) + C$ for an additive constant $C$. By construction, $f^\star_C$ is differentiable and thus $f^\star_C \in \mathcal F$. Moreover,
for any $f \in \mathcal{F}$, we have, since $\nabla f$ is measurable,
\begin{equation}
\begin{aligned}
   \norm{\nabla_\theta E_\psi(x,\theta) + \nabla f(\theta)}^2  &\geq \min_{g \text{ measurable}} \norm{\nabla_\theta E_\psi(x,\theta) + \nabla f^*_C(\theta)}^2
                                                               &\geq \norm{\nabla_\theta E_\psi(x,\theta) + \nabla f^\star(\theta)}^2
\end{aligned}
\end{equation}
Making all $f^\star_C$ the minimizers of  \cref{prop: cond-exp-l2-opt}'s problem.
\end{proof}

\begin{lemma}[Differentiablity of the log-normalizer] \label{lem:diff-log-z}
Let $\mathcal X$ and $\Theta$ be two open sets of $\mathbb R^{d_x}$ and $\mathbb R^{d_\theta}$. Assume that $E_\psi(x, \theta)$ is differentiable for all $(\theta,x)\in\Theta \times \mathcal X$. Then the map
\begin{equation}
   \begin{aligned}
    \theta  \longmapsto \log Z(\theta) := \log \int_{\mathcal Z} e^{-E_\psi(x, \theta)}\textrm{d}x
    \end{aligned}
\end{equation}
is differentiable, and its  derivative is given by: $\nabla \log Z(\theta) = -\int \nabla_\theta E_\psi(x, \theta) e^{-E_\psi(x, \theta)}\textrm{d}x$.
\end{lemma}
\begin{proof}
   We first prove the differentiability of $\log Z$ and then derive the form of its derivative. The first part of the proof borrows inspiration from Theorem 2.2 of \cite{brown1986fundamentals} which proves the result only in the case of exponential families. Let $\theta_0 \in \Theta$.
   Since $\Theta$ is open, there exists a open ball $B(\theta_0, \epsilon)$ of radius $\epsilon$ centered at $\theta_0$ contained in $\Theta$. Consider the restriction of $Z(\cdot)$ to $B(\theta_0, \epsilon)$. Then for all $i \in 1,\dots,d_\theta$, for all $\theta \in B(\theta_0, \epsilon), x\in\mathcal{X}$, $|\frac{\partial E_\psi(x, \theta)}{\partial \theta_i} e^{-E_\psi(x, \theta)}| \leq \sup_{\theta \in \mathcal B(\theta_0, \epsilon) }|\frac{\partial E_\psi(x, \theta)}{\partial \theta_i} e^{-E_\psi(x, \theta)}| < \infty $ for all $x$, since 
   $\theta \longmapsto \frac{\partial E_\psi(x, \theta)}{\partial \theta_i} e^{-E_\psi(x, \theta)}$ is continuous on $\Theta$, and thus bounded on $\overline{B(\theta_0,\epsilon)}$. By the dominated convergence theorem, we can now differentiate the function $Z: \theta \longmapsto \int e^{-E_\psi(x, \theta)} \textrm{d}x$ under the integral sign for any $i$ to compute the gradient $\nabla_\theta \log Z(\theta) = \frac{\nabla_\theta \int e^{-E_\psi(x,\theta)} \text{d}x}{Z(\theta)}$. Since $Z(\theta)>0$ for any $\theta$, $\log Z(\cdot)$ is differentiable, and its gradient is given by:
\begin{align}
    \nabla_\theta \log Z(\theta) &= \frac{\nabla_\theta \int e^{-E_\psi(x,\theta)} \text{d}x}{Z(\theta)} \nonumber\\
                                      &= \int \frac{\nabla_\theta e^{-E_\psi(x,\theta)}}{Z(\theta)} \text{d}x  \nonumber\\ 
                                      &= \int \frac{(-\nabla_\theta E_\psi(x,\theta)) e^{-E_\psi(x,\theta)}}{Z(\theta)} \text{d}x  \nonumber\\
                                      &= -\mathbb{E}_{x \sim q_\psi(x|\theta)} \left[ \nabla_\theta E_\psi(x,\theta) \right]. \label{eq:grad-log-z-cond-exp}
\end{align}
\end{proof}


\subsection{Empirical improvements to DIVI} \label{app-subsec:emp-improvements-divi}
We propose a few improvements to the DIVI method outlined in \cref{alg:divi}.

{\bf Choice of $\nu$.} The DIVI method allows for any choice $\nu$ of proposal on $\theta$. In practice, we set $\nu(\theta) = q_{\psi_r^*}(\theta|x_o)$, the previous round's posterior estimate. By doing so, we concentrate the log-Z network training data around parameters most relevant to the observation $x_o$, ensuring that our $\lz_\eta$ is accurate on the regions of parameter space that are most relevant to the problem at hand.

{\bf Variance reduction.} As detailed in \cref{eq:grad-log-z-cond-exp}, the training signal for $\nabla_\theta \lz_\eta$ is given by data points $\left\{\theta^{(i)}, -\mathbb{E}_{x\sim q_\psi(x|\theta^{(i)})} [\nabla_\theta E_\psi(x,\theta^{(i)})]\right\}_i$. The version of DIVI in \cref{alg:divi} effectively approximates this conditional expectation with an empirical one-sample estimate: $-\nabla_\theta E_\psi (x^{(i)}, \theta^{(i)})$, for $x^{(i)} \sim q_\psi(\cdot|\theta^{(i)})$. We can reduce the variance of this estimate by sampling multiple points from the likelihood. The approximation then becomes
\begin{equation*}
    -\frac{1}{M}\sum_{m=1}^M E_\psi(x^{(i)}_m,\theta^{(i)}), \qquad x^{(i)}_m \overset{\text{iid}}{\sim} q_\psi(\cdot|\theta^{(i)}).
\end{equation*}

{\bf Hyperparameter tuning.} For all experiments in the main paper and appendix, we use the same set of hyperparameters, with the exception of in the following problems:
\begin{itemize}
    \item Gaussian Linear Uniform: \texttt{max\_iter=10 (default: 500)}. We reduce the number of iterations of EBM training in order to avoid overfitting, because the true likelihood of this model is a very simple multivariate Gaussian \citep{lueckmann2021benchmarking}. We do this only when \texttt{num\_samples==100 or 1000}.
    \item Lotka-Volterra: \texttt{learning\_rate=0.001 (default: 0.01)}.
    \item Pyloric: \texttt{learning\_rate=0.0001 (default: 0.01)}.
\end{itemize}

{\bf Training the log-normalizer network in parallel to EBM training.} See \cref{app-subsec:online-lznet-training}.

\subsection{Online log-Z network training in SUNLE} \label{app-subsec:online-lznet-training}
We now describe how data used in producing particle approximations during EBM training can be recycled to train the log-Z network online. \cref{alg:cebm-ml-training} generates particle approximations targeting the current likelihood $q_{\psi_k}$, during which it uses existing samples $\theta^i$ and generates $x^i_m$ approximately distributed as $q_{\psi_k}(\cdot|\theta^i)$. Let \texttt{make\_particle\_approx\_recycled\_data} refer to an augmented version of \texttt{make\_particle\_approx} that returns not only a particle approximation $\widehat{q}$, but also the particles themselves. \cref{alg:cebm-and-lznet-training} details a variant of \cref{alg:cebm-ml-training} that uses these new samples to update the log-Z network.

\begin{algorithm}[H] \caption{$\texttt{maximize\_cebm\_log\_l\_and\_train\_log\_z}(\mathcal D, \psi_0, \eta_0)$}\label{alg:cebm-and-lznet-training}
	\begin{algorithmic}
		\STATE \hspace{-1em}\textbf{Input:} training data $ \mathcal D:=\{ \theta^{(i)}, x^{(i)} \}_{i=1}^{N} $, initial EBM parameters $ \psi_0 $, initial log-Z network parameters $\eta_0$
		\STATE \hspace{-1em}{\bf Output:} conditional density estimator $ q_{\psi}(x | \theta) $, log-Z network $\lz_\eta(\cdot, \psi)$
		\STATE \hspace{-1em}{\bf Initialize} $q_{\psi_0} \propto e^{-E_{\psi_0}(\theta, x)}, \{\hat{q}_i = \delta_{x^i}\}_{i=1}^{N}$
		\STATE \hspace{-1em}{{\bf for} $k=0,\dots,K-1$ {\bf do}}
		\STATE \hspace{0.5em}\hspace{-1em}{{\bf for} $i=0,\dots,N-1$ {\bf do}}
			\STATE \hspace{-1em}\hspace{0.5em}\hspace{0.5em} $ \widehat{q}_i, \{x^i_m\}_{m=1}^M \coloneqq \verb+make_particle_approx+(q_{\psi_k}(\cdot, \theta^i), \widehat{ q }_i)$
            \STATE \hspace{-1em}\hspace{0.5em}\hspace{0.5em} $\widehat{L} = -\frac{1}{M}\sum_{m=1}^M \norm{\nabla_\theta E_{\theta_k}(x^i_m, \theta)|_{\theta=\theta^i} - \nabla_\eta\lz_\eta(\theta^i, \psi_k)|_{\eta=\eta_i}}^2$
            \STATE \hspace{-1em}\hspace{0.5em}\hspace{0.5em} $\eta_{i+1} = \verb+ADAM+(\eta_i, \widehat{L})$
		\STATE \hspace{-1em}\hspace{0.5em}{{\bf end for}}
		\vspace{0.15em}
		\STATE \hspace{-1em}\hspace{0.5em}$ \widehat{ G }\hspace{-0.3em} = \hspace{-0.1em} -\frac{1}{N}\sum \nabla_{ \psi } \hspace{-0.1em} E_{\psi_k}(x^{i}, \theta^{i}) +\hspace{-0.15em}\mathbb{E}_{ \widehat{ q }_i } \nabla_{ \psi }  E_{\psi_k} (x^i, \theta^{i})$
		\STATE \hspace{-1em}\hspace{0.5em}$ \psi_{k+1} \hspace{-0.3em} = \verb+ADAM+(\psi_k, \widehat{ G })$
		\STATE \hspace{-1em}{{\bf end for}}
		\STATE \hspace{-1em}{\bf Return} $ q_{\psi_K}$
	\end{algorithmic}
\end{algorithm}
\hfill

The above update steps in $\eta$ can be used in conjunction with the ``standard'' $\eta$ updates in \cref{alg:s-unle} to improve log-Z network accuracy, particularly in difficult problems where the true log-normalizer exhibits pathological behavior.

\newpage
\section{Additional Experimental and Inferential Details} \label{app-sec:exp-details}

\subsection{Posterior pairplots on benchmark Problems} \label{app-subsec:posteriors}

We report the ground truth estimated posterior pairplots on benchmark problems in \cref{fig:benchmark-densities}. AUNLE and SUNLE exhibit satisfying mode coverage, and are able to capture complex posterior structures.

\begin{figure}[H]
    \centering
	\hspace{-1.5em} \includegraphics[width=.25\textwidth]{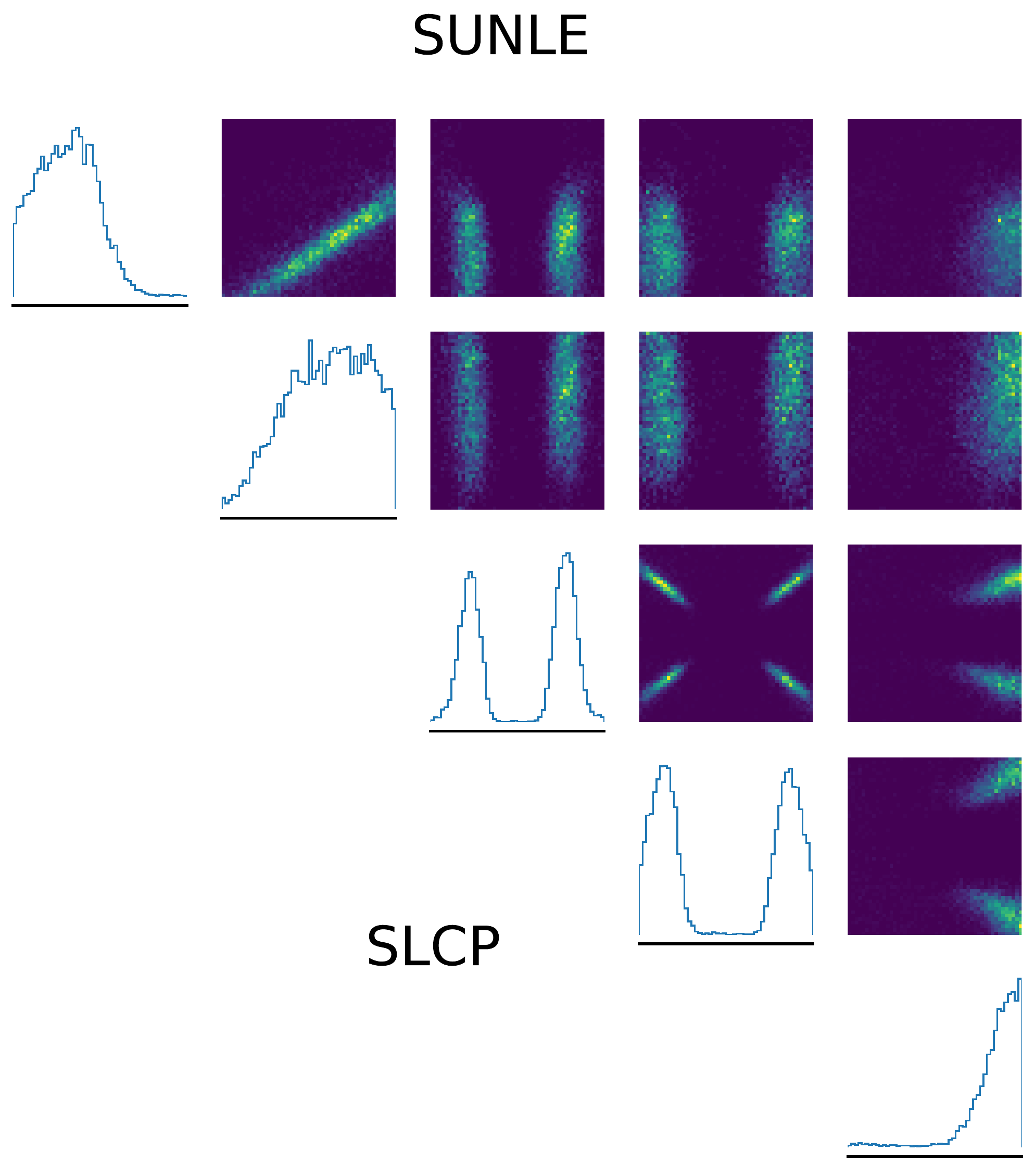} \hspace{-1em} %
	\includegraphics[width=.25\textwidth]{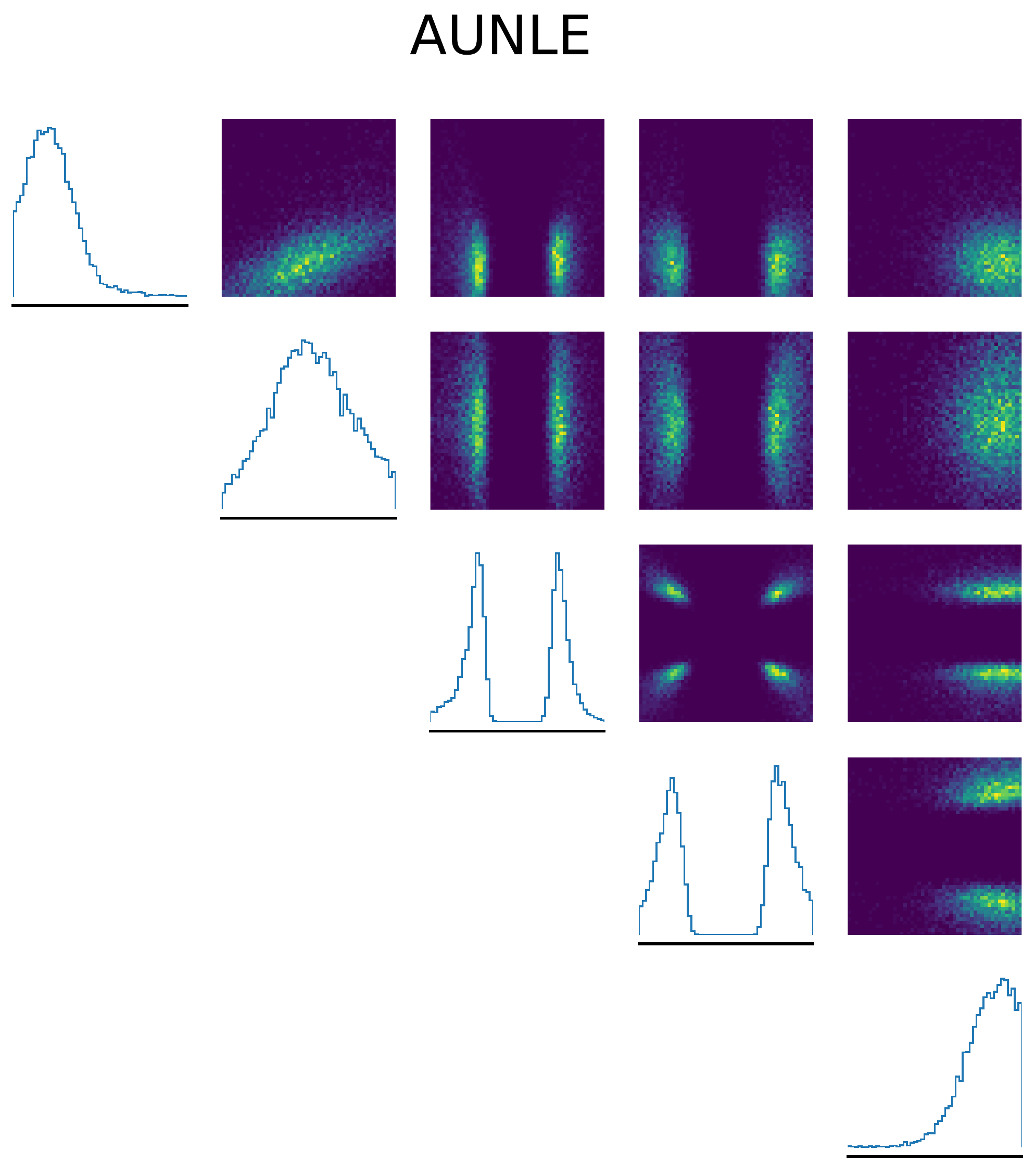} \hspace{-1em} %
	\includegraphics[width=.25\textwidth]{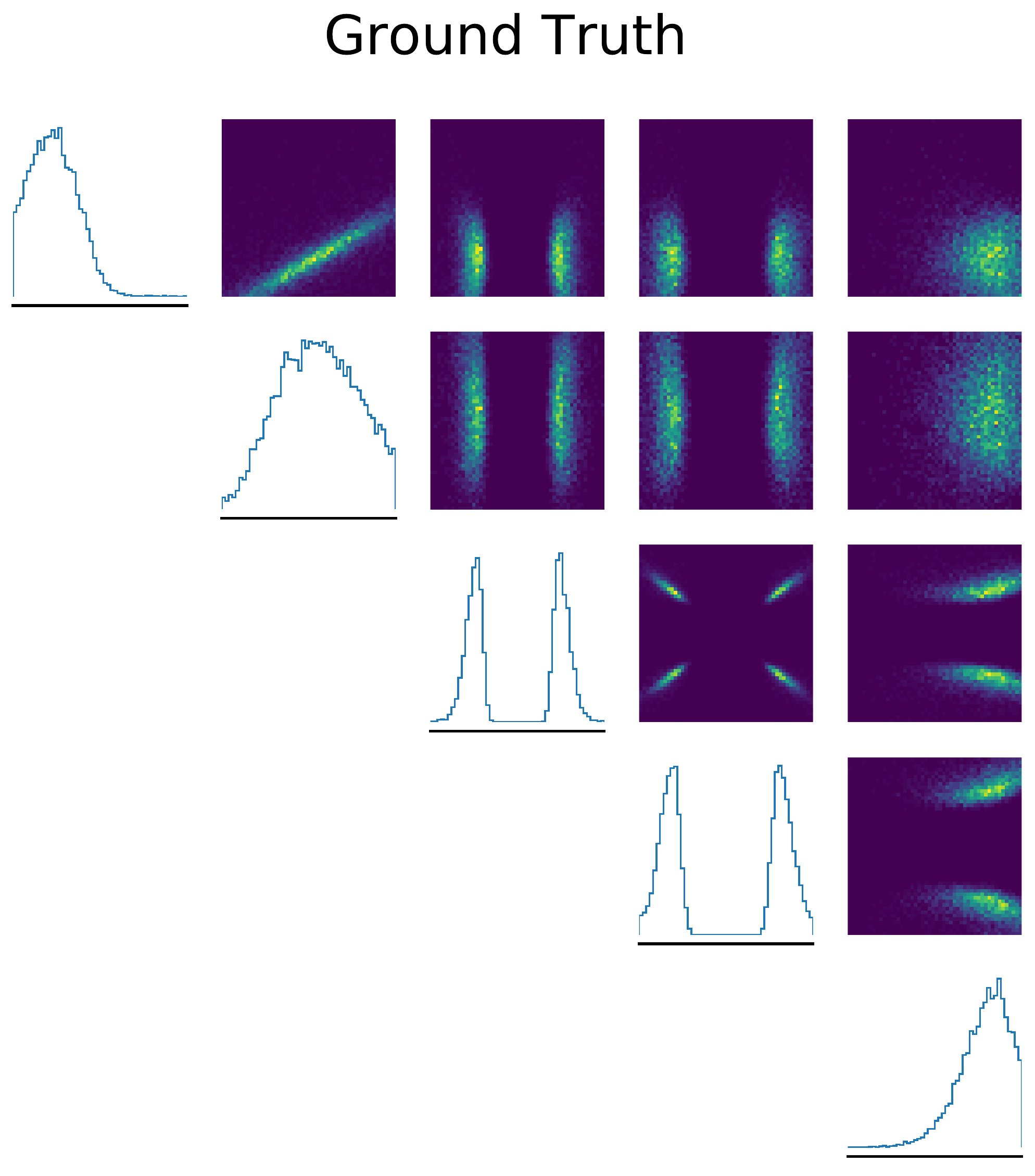}
	
	\hspace{-1.5em} \includegraphics[width=.25\textwidth]{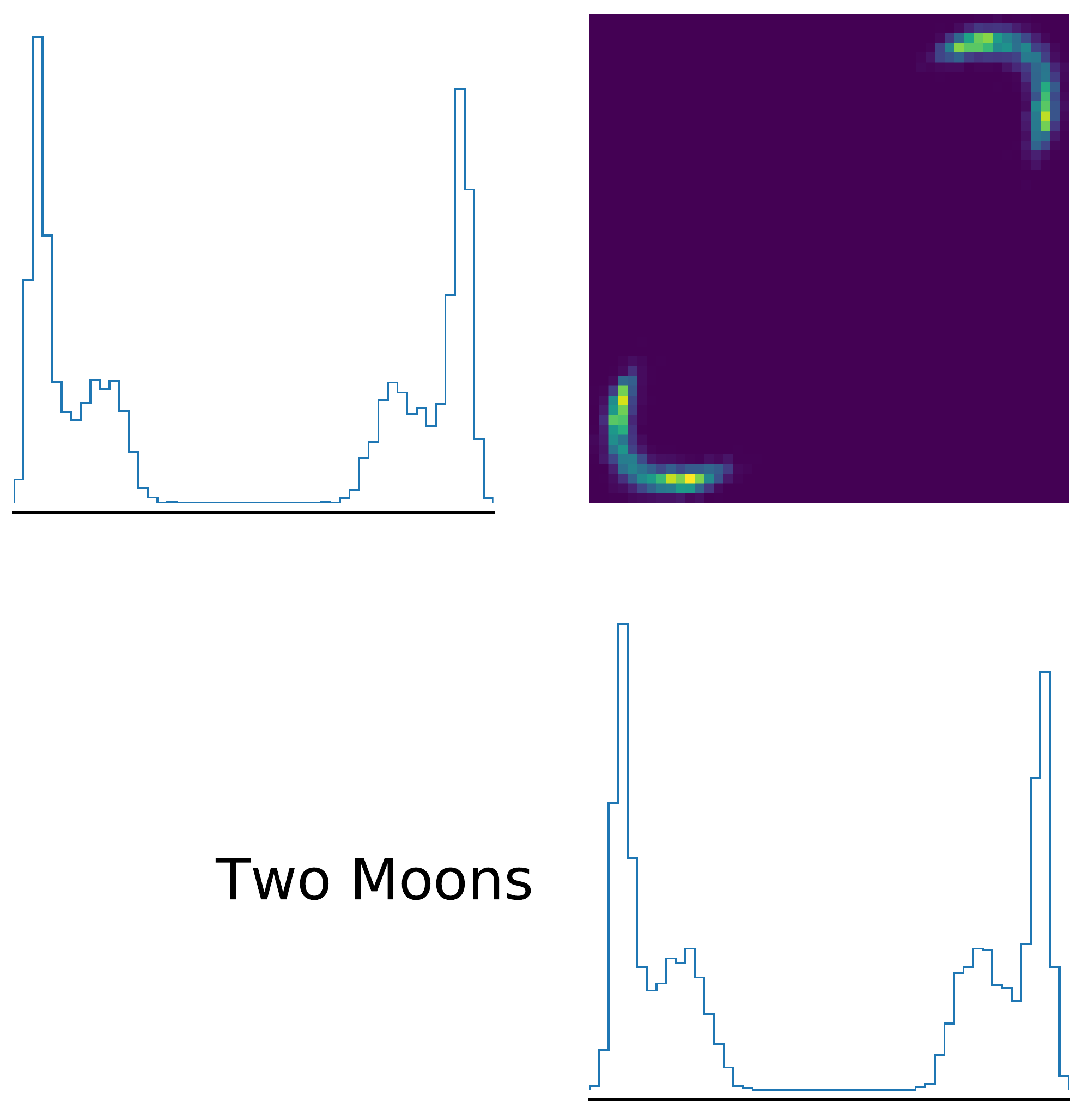} \hspace{-1em} %
	\includegraphics[width=.25\textwidth]{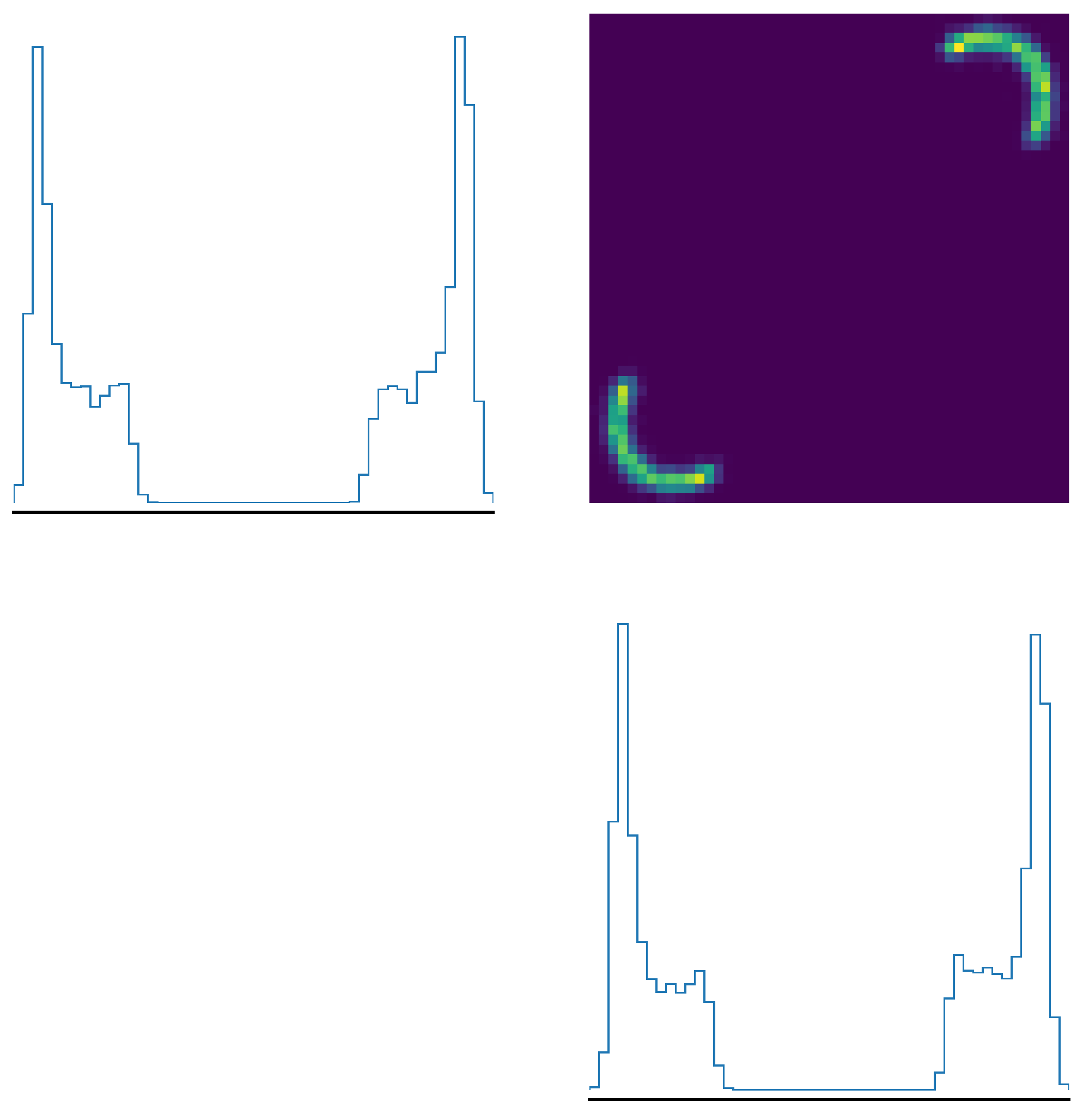} \hspace{-1em} %
	\includegraphics[width=.25\textwidth]{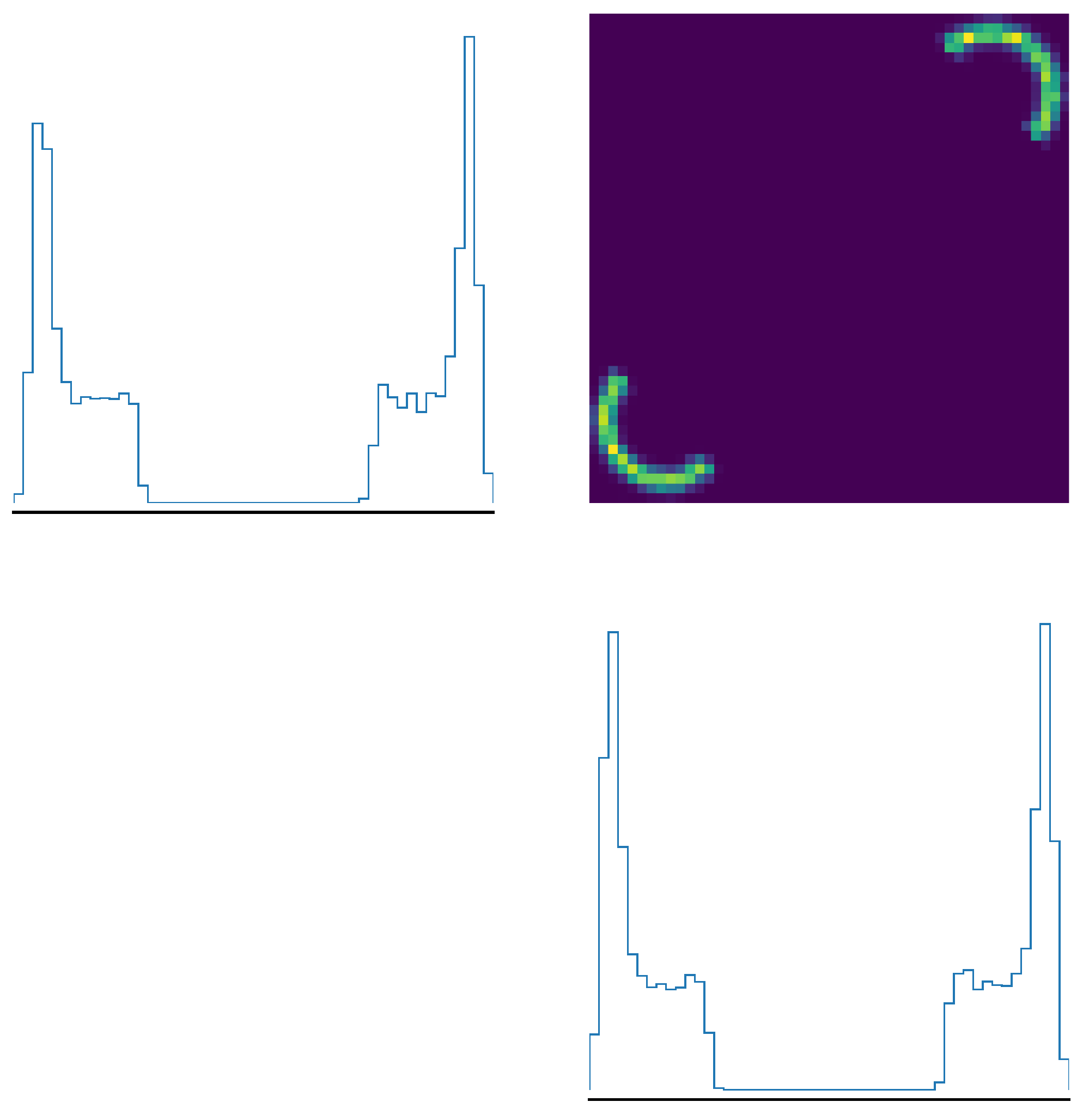}
	
	\hspace{-1.5em} \includegraphics[width=.25\textwidth]{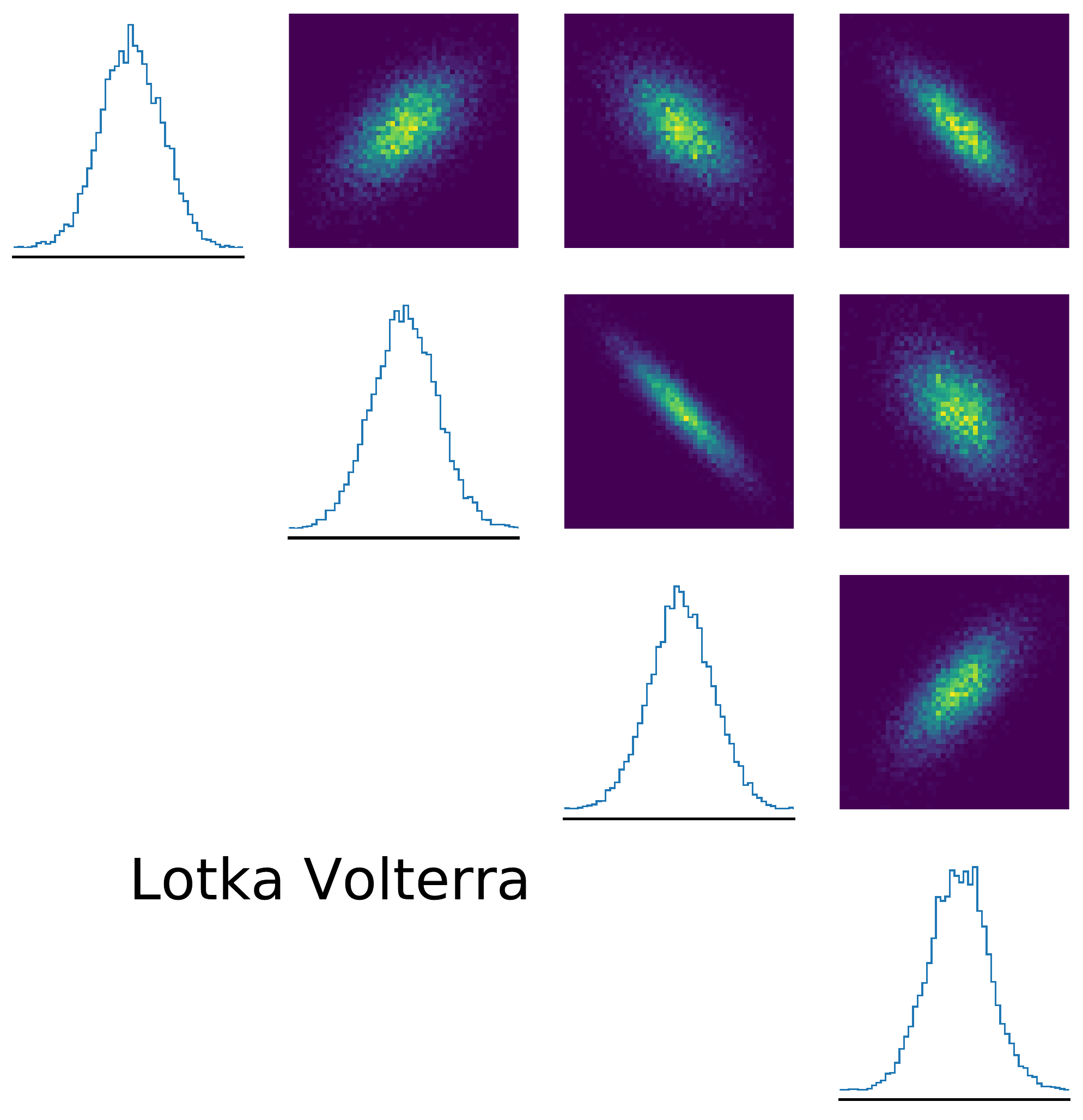} \hspace{-1em} %
	\includegraphics[width=.25\textwidth]{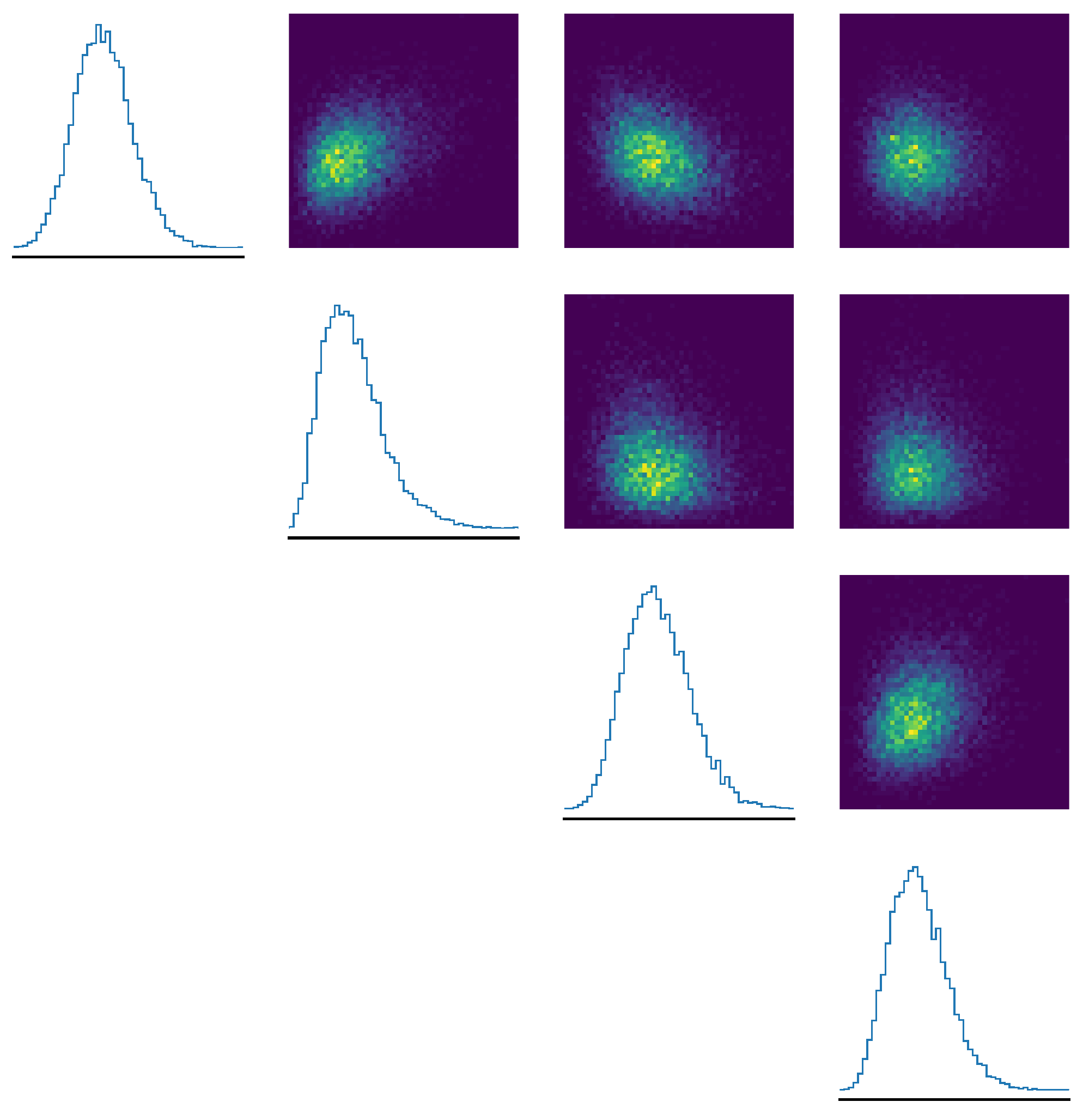} \hspace{-1em} %
	\includegraphics[width=.25\textwidth]{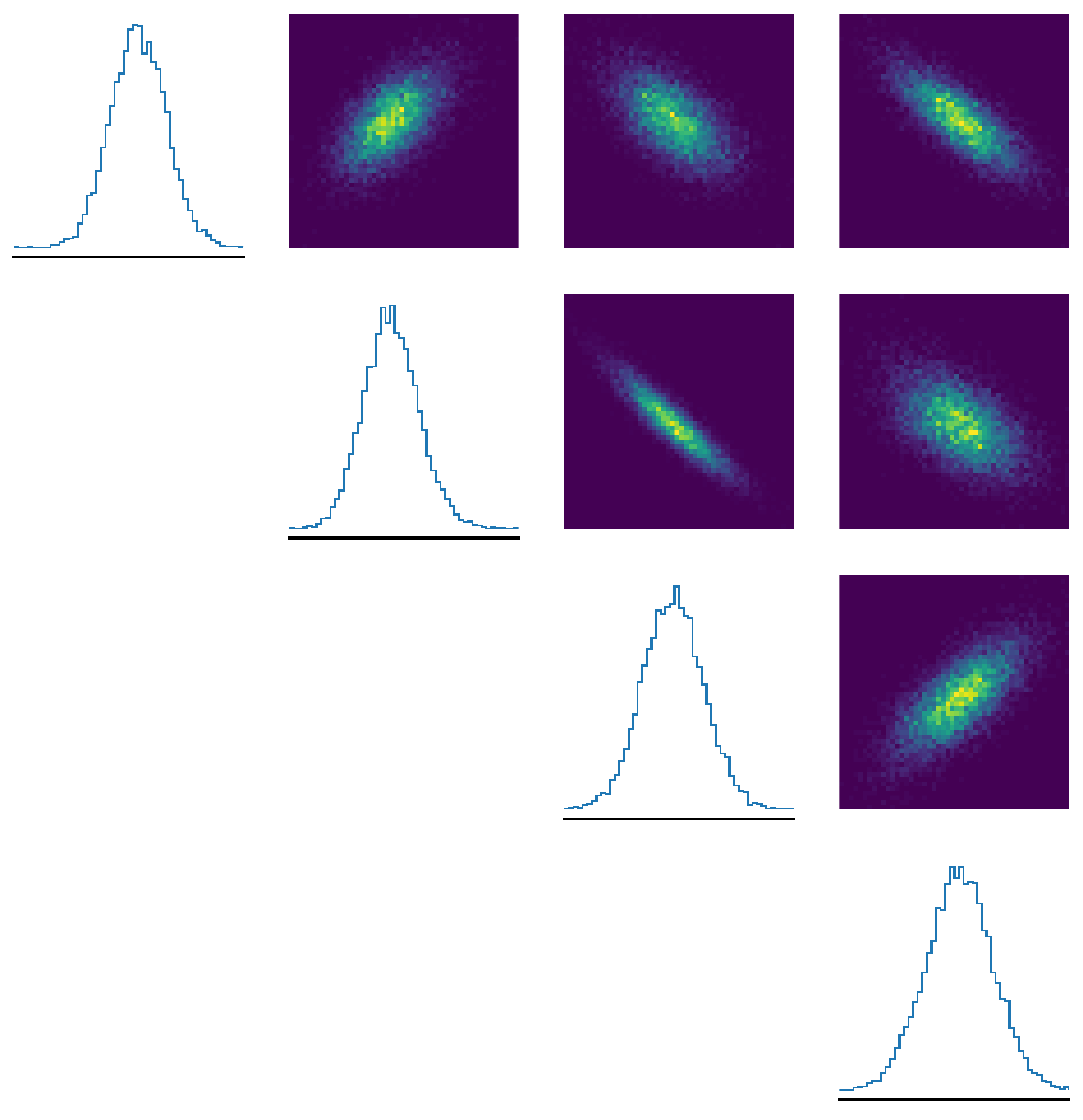}
	
	\hspace{-1.5em} \includegraphics[width=.25\textwidth]{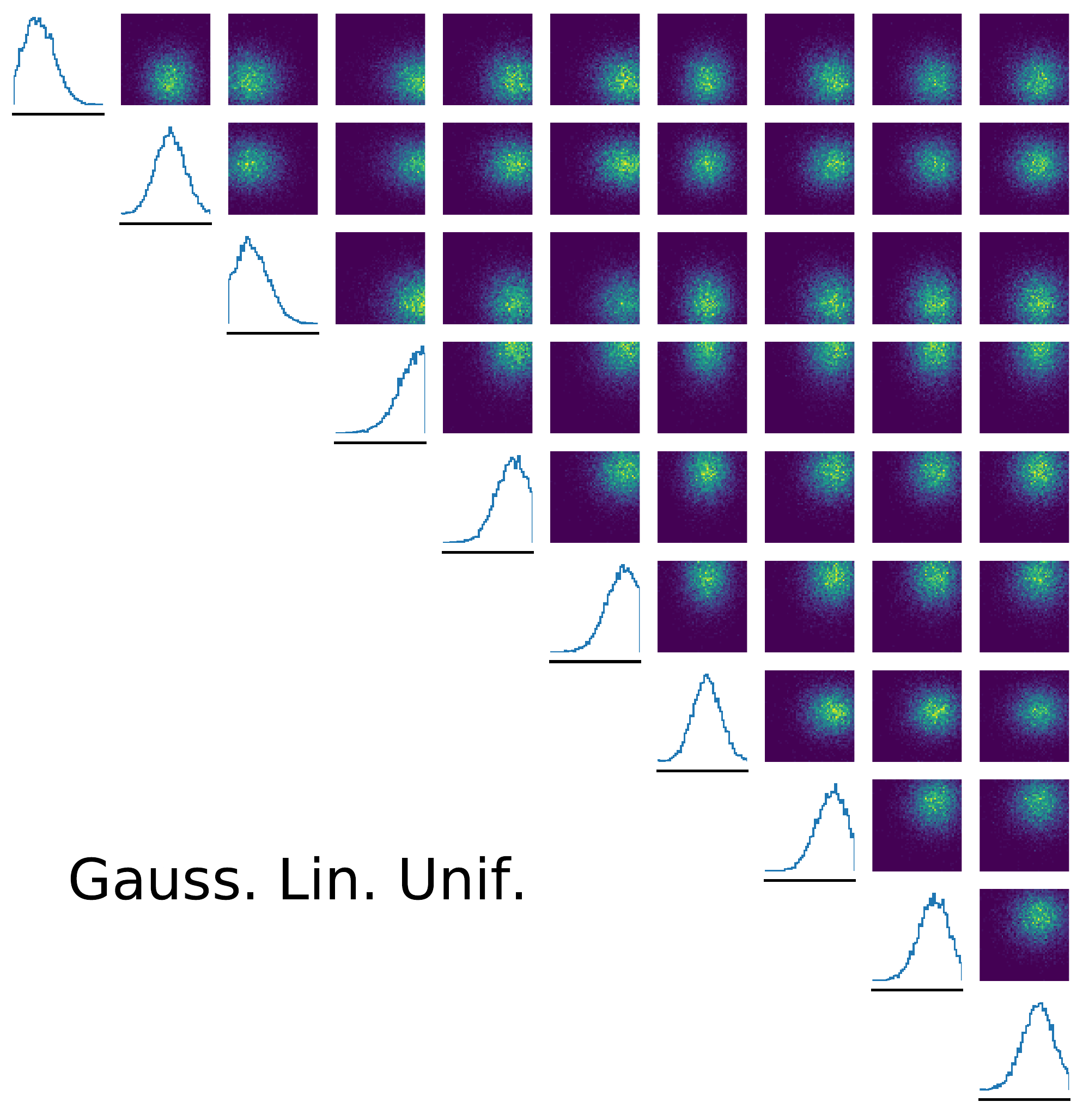} \hspace{-1em} %
	\includegraphics[width=.25\textwidth]{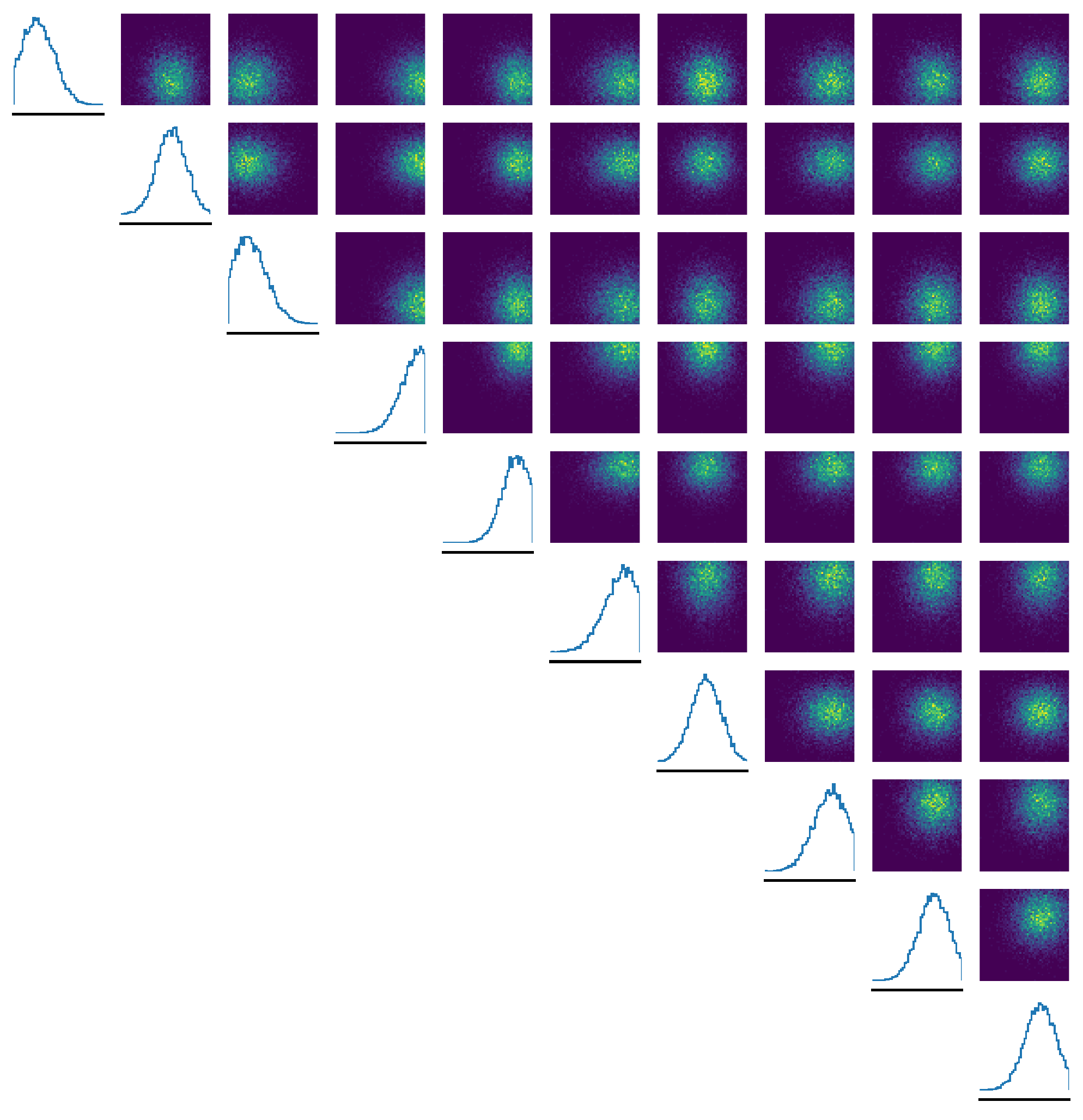} \hspace{-1em} %
	\includegraphics[width=.25\textwidth]{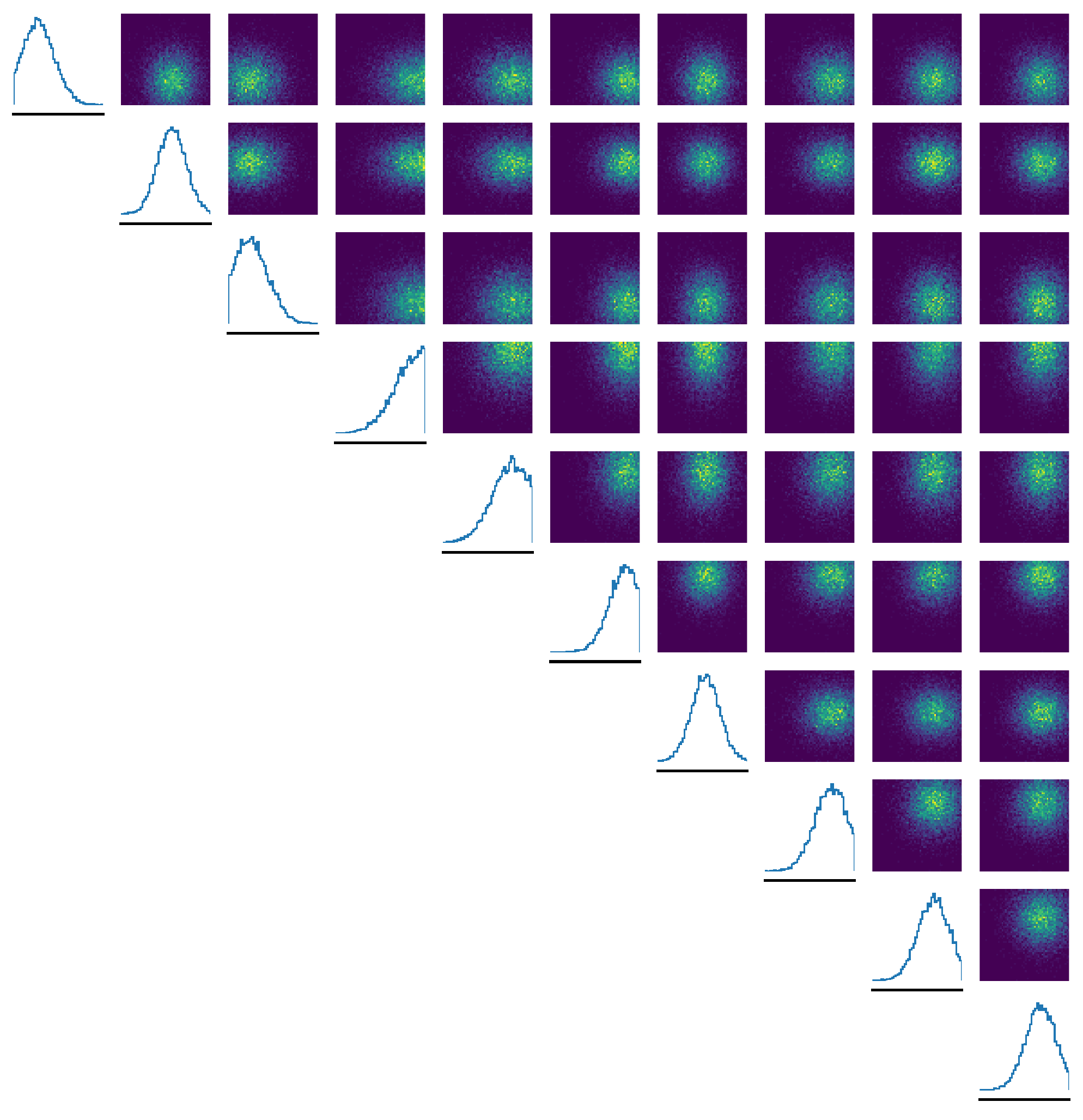}
	
	\caption{Posterior marginal (empirical) pairplots for SUNLE's posterior (first column), AUNLE's posterior (second column) and the ground truth posterior for the four studied benchmark problems. Each row outlines a separate benchmark problem.}
    \label{fig:benchmark-densities}
    \vspace{-1em}
\end{figure}

\subsection{Manifestation of the short-run effect in UNLE} \label{app-subsec:short-run}

It was shown in \cite{nijkamp2019learning}
that EBMs trained by replacing the intractable expectation under the EBM with an expectation under a particle approximation obtained by running parallel runs of Langevin Dynamics initialized from random noise and updated for a fixed (and small) amount of steps can yield an EBM whose density is not proportional to the true density, but rather a generative model that can generate faithful images by running a few steps of Langevin Dynamics from random noise on it.
Our design choices for both training and inference purposefully avoid this effect from  manifesting in UNLE. During training, we estimate the intractable expectation using \emph{persistent} MCMC or SMC chains, e.g by initializing the MCMC (or SMC) algorithm of iteration $k$ with the result of the MCMC (or SMC) algorithm at iteration $k-1$, yielding a different training method than short-run EBMs. At inference, the posterior model is sampled from Markov Chains with a significant burn-in period, contrasting with the sampling model of short-run EBMs.
\cref{fig:two-moons-density} compares the density of UNLE's posterior estimate for the Two Moons model (a 2D posterior which can be easily visualized) with the true posterior. As \cref{fig:two-moons-density} shows, AUNLE and SUNLE's posterior density match the ground truth very closely, demonstrating that UNLE's EBM is not a short-run generative model, but a faithful \emph{density estimator}.

\begin{figure}[H]
	\includegraphics[width=\textwidth]{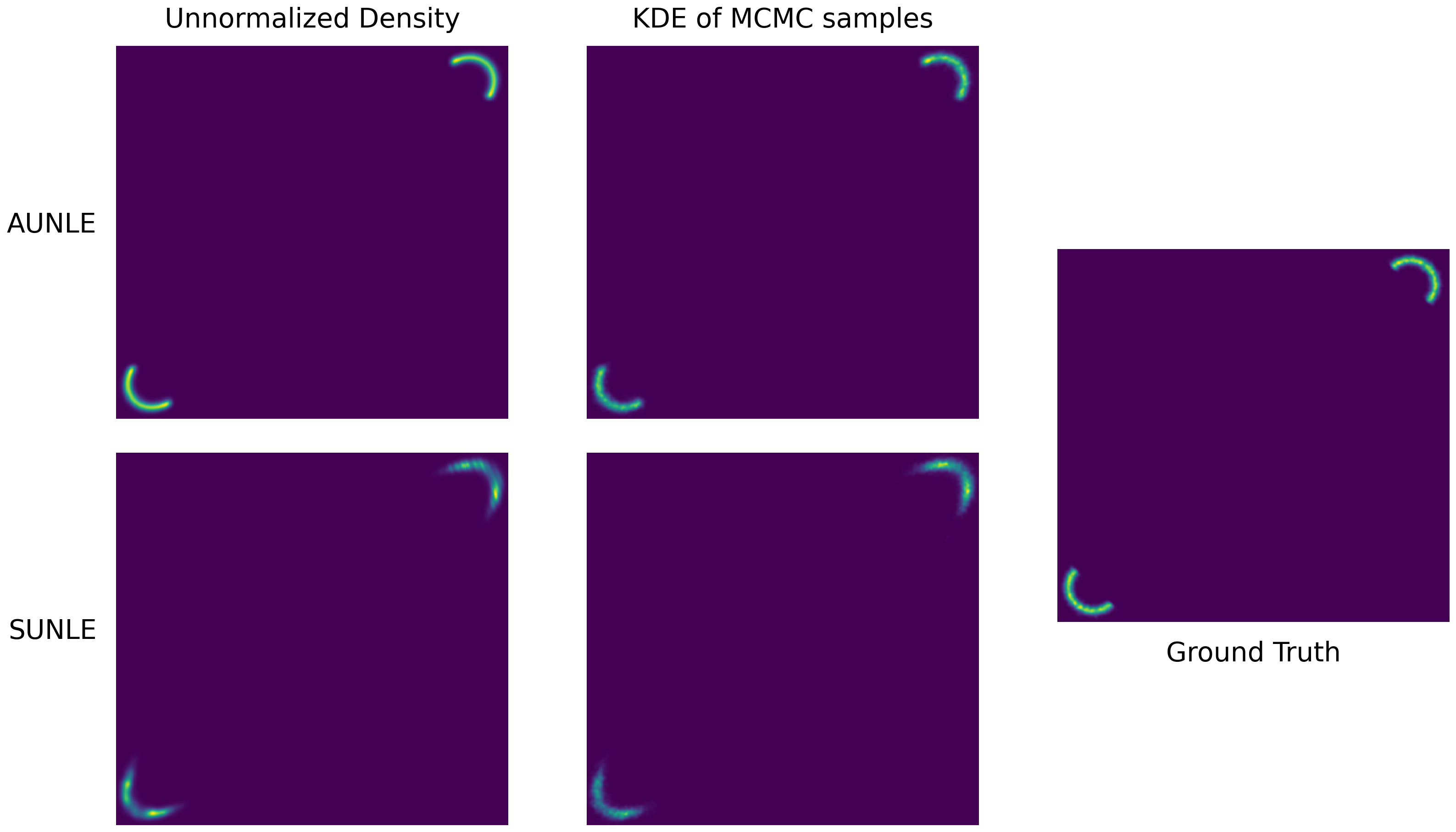}
	\caption{Normalized posterior densities of AUNLE and SUNLE for the Two Moons model. Left: manually normalized posterior densities of AUNLE and SUNLE using a discretization of the posterior over a grid. Middle: kernel density estimation of the MCMC samples obtained from AUNLE's and SUNLE's posteriors. Right: Ground Truth posterior. AUNLE's and SUNLE's posterior densities closely match the true density, showing that these methods indeed learn a density estimator and a generative model \citep{nijkamp2019learning}.}
	\label{fig:two-moons-density}
    \vspace{-1em}
\end{figure}

\subsection{Validating the $(Z, \theta)$-uniformization of AUNLE's posterior in practice} \label{app-sec:tilting-validation}

\cref{prop:aunle-auto-normalization} ensures that the normalizing constant $Z(\theta, \psi)$ present AUNLE's posterior is independent of $\theta$ \emph{provided that the problem is well-specified, and that $\psi = \psi^\star$}, the optimum of AUNLE's population objective. In practice, these conditions will not hold exactly, and the uniformization of AUNLE's posterior thus only holds approximately. To assess the loss of precision associated with using a standard MCMC posterior in the context of approximate uniformization, we compare the quality of AUNLE's posterior samples obtained using a standard MCMC sampler (which is valid only if uniformization holds), and using a doubly-intractable MCMC sampler, which handles non-uniformized posteriors. We mitigate the approximation error of doubly-intractable samplers by using a large number of steps (1000) when sampling from the likelihood using MCMC. As \cref{fig:doubly-vs-singly-intractable} shows, there is no gain in using a doubly-intractable sampler for inference in AUNLE, suggesting that the uniformization property of AUNLE holds well in practice.

\begin{figure}[H]
	\includegraphics[width=\textwidth]{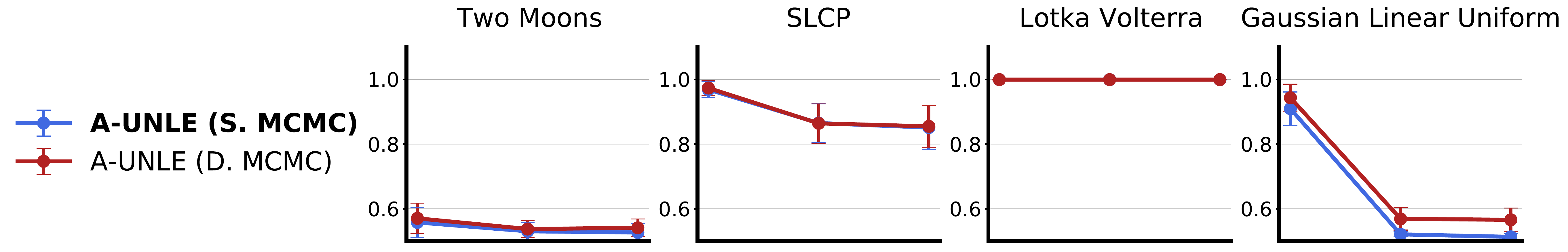}
	\caption{Quality of AUNLE's posterior samples (measured in classifier accuracy) obtained using a standard MCMC sampler (S. MCMC) and a doubly-intractable sampler (D. MCMC). The results show no gain in using a doubly-intractable sampler, justifying the use of standard samplers for AUNLE.}
	\label{fig:doubly-vs-singly-intractable}
    \vspace{-1em}
\end{figure}

\subsection{Computational Cost Analysis} \label{app-sec:computational-cost}

Training unnormalized models using approximate likelihood is computationally
intensive, as it requires running a sampler during training at each gradient
step, yielding a computational cost of $O(T_1 T_2 N)$, where $T_1$ is the number
of gradient steps, $T_2$ is the number of MCMC steps, and $N$ is the number of
parallel chains used to estimate the gradient.

To maximize the efficiency of training, we implement all samplers using
\texttt{jax} \cite{frostig2018compiling}, which provides a just-in-time compiler and an auto-vectorization
primitive that generates efficient, custom parallel sampling routines. For
AUNLE, we introduce a warm-started SMC approximation procedure to estimate gradients,
yielding competitive performance with as little as 5 intermediate
probabilities per gradient computation. For SUNLE, we warm-start the parameters
of the EBM across training rounds, and warm-start the chains of the doubly-intractable sampler across inference rounds, which significantly reduces the
need for burn-in steps and long training. Finally, all experiments are run on
GPUs. Together, these techniques make AUNLE and SUNLE almost always the fastest methods for amortized and sequential inference, with total per-problem runtimes of 2 minutes for AUNLE and 15 minutes for SUNLE on benchmark models (which is significantly faster than NLE and SNLE on their canonical CPU setup, \citealt{lueckmann2021benchmarking}) and less than 3 hours for SUNLE on the
pyloric network model (with half of this time spent simulating samples). The latter is \emph{10 times faster than SNVI} (30 hours) on the same model. A breakdown of training, simulation and inference
time is provided in \cref{fig:time_breakdown}. We note that (S)NLE was run on a CPU, which is the advertised computational setting \citep{lueckmann2021benchmarking}, as (S)NLE uses deep and shallow networks that do not benefit much from GPU acceleration.

We note that the time spent performing inference is negligible for AUNLE, which uses standard MCMC for inference thanks to the tilting trick employed in its model. On the other hand, the runtime of SUNLE, which performs inference using a doubly-intractable sampler is dominated  by its inference phase. This point demonstrates the computational benefits of the AUNLE's tilting trick. Note that SUNLE performs inference in a multi-round procedure, and requires thus $R$ training and inference phases (where $R$ is the number of rounds), as opposed to 1 for AUNLE. We alleviate this effect by leveraging efficient warm-starting strategies for both training and inference, which to an extent amortizes these steps across rounds.


%
%
%

\begin{figure}[H]
	\includegraphics[width=\textwidth]{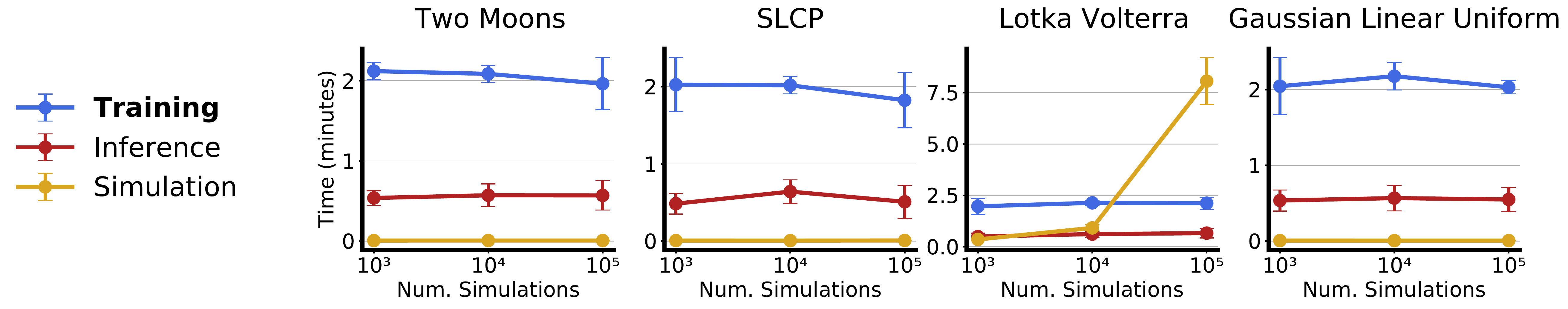}
	\includegraphics[width=\textwidth]{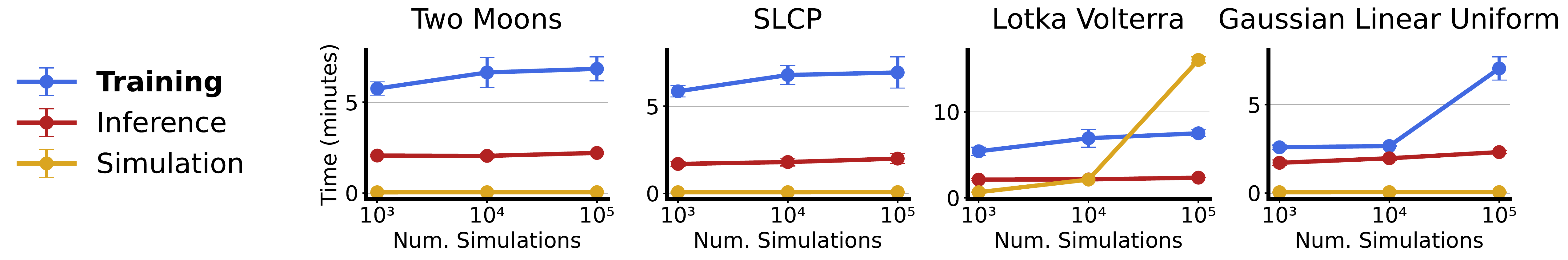}
	\includegraphics[width=\textwidth]{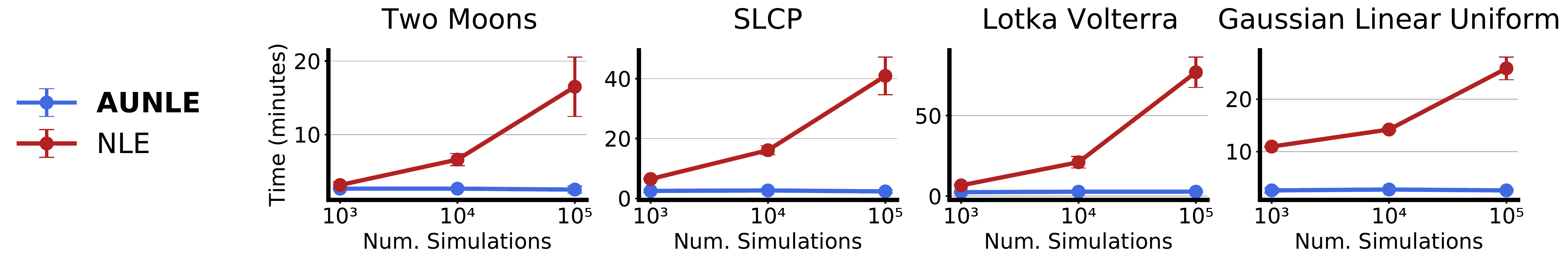}
	\includegraphics[width=\textwidth]{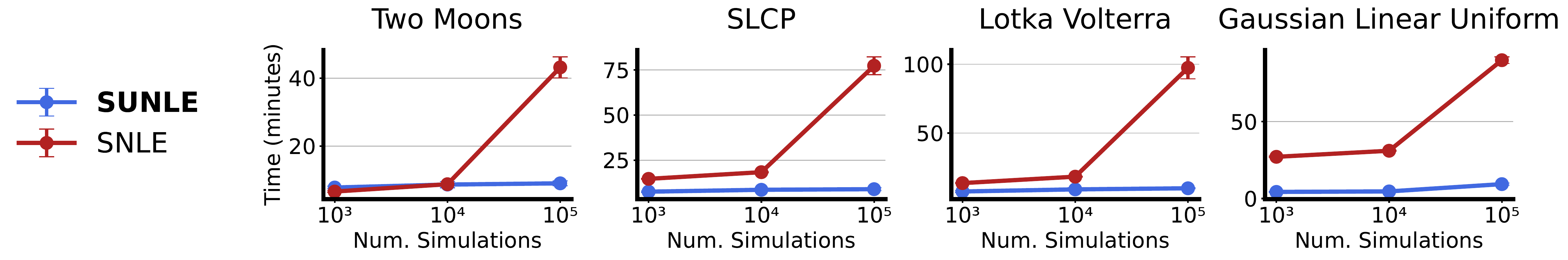}
	\caption{Runtime of UNLE: Analysis and Comparisons. First row: time (in minutes) spent training, inferring, and simulating for AUNLE. Second row: time (in minutes) spent training, inferring, and simulating for SUNLE. Third row: runtime comparison between AUNLE and NLE (in log-scale). Fourth row: runtime comparison between SUNLE and SNLE.}
	\label{fig:time_breakdown}
    \vspace{-1em}
\end{figure}

%
%
%

\subsection{Experimental setup for SNLE and SMNLE} \label{app-subsec:experimental-setup}
\paragraph{SNLE}
The results reported for SNLE are the one present in the SBI benchmark suite \citep{lueckmann2021benchmarking}, which reports the performance of both NLE and SNLE on all benchmark problems studied in this paper.
\paragraph{SMNLE}
The results reported for SMNLE were obtained by running the implementation referenced by \cite{pacchiardi2020score}.
SMNLE comes in two variants: the first variant uses standard Score Matching \citep{hyvarinen2005estimation} to
estimate its conditional EBM, while the second variant uses Sliced Score
Matching \citep{song2020sliced}, which yields significant computational speedups during training. For both methods, we
train the model using 500 epochs, and neural networks with 4 hidden layers and
50 hidden and outputs units.  To optimize the inference performance, we carry
out inference using our own doubly-intractable sampler, which automatically
tunes all parameters of the doubly-intractable samplers except for the
number of burn-in steps, and initializes the chain at local posterior modes. We carry out a grid search over the learning rates 0.01 and 0.001, and leave other training parameters to their default. Figures in the main body only report the performance of the Sliced Score Matching variant, which perform better in practice and run faster by an order of magnitude. \cref{fig:sm-ssm} reports the performance of both variants for completeness. We used GPUs for both training and inference in SMNLE, yielding similar or higher training compared to AUNLE when using the sliced variant, and much longer training times when using the standard variant.

\begin{figure}[htbp]
	\includegraphics[width=\textwidth]{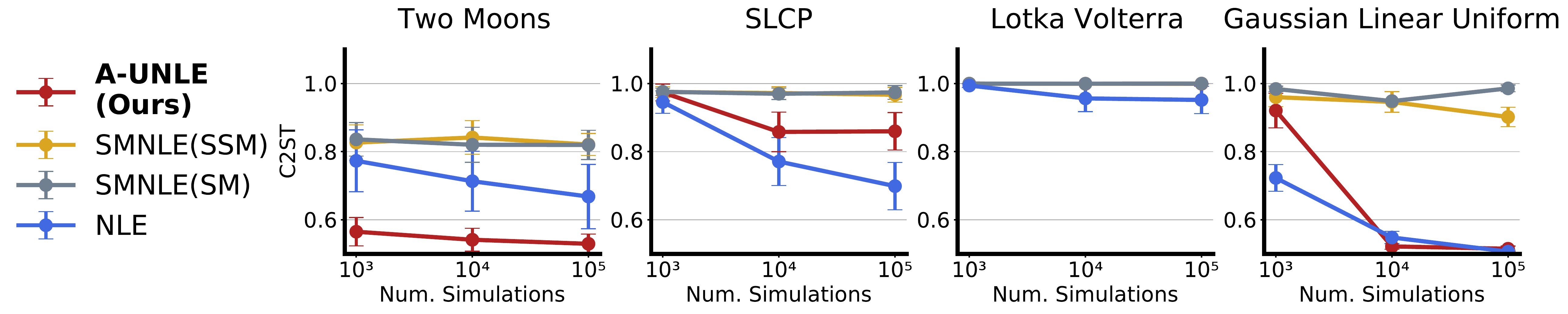}
	\caption{Comparison of AUNLE against SMNLE with Sliced Score Matching (SSM), SMNLE with Score Matching (SM), and NLE on a set of benchmark problems.}
	\label{fig:sm-ssm}
    \vspace{-1em}
\end{figure}


\subsection{Neuroscience Model: Details} \label{app-sec:neuroscience}

\paragraph{Pairwise Marginals}
We provide the full pairwise marginals obtained after computing a kernel density estimation on the final posterior samples of SUNLE. We retrieve similar patterns as the one displayed in the pairwise marginals of SNVI samples. We refer to \cite{glockler2021variational} for more details on the specificities of this model.

\begin{figure}[htbp]
	\includegraphics[width=\textwidth]{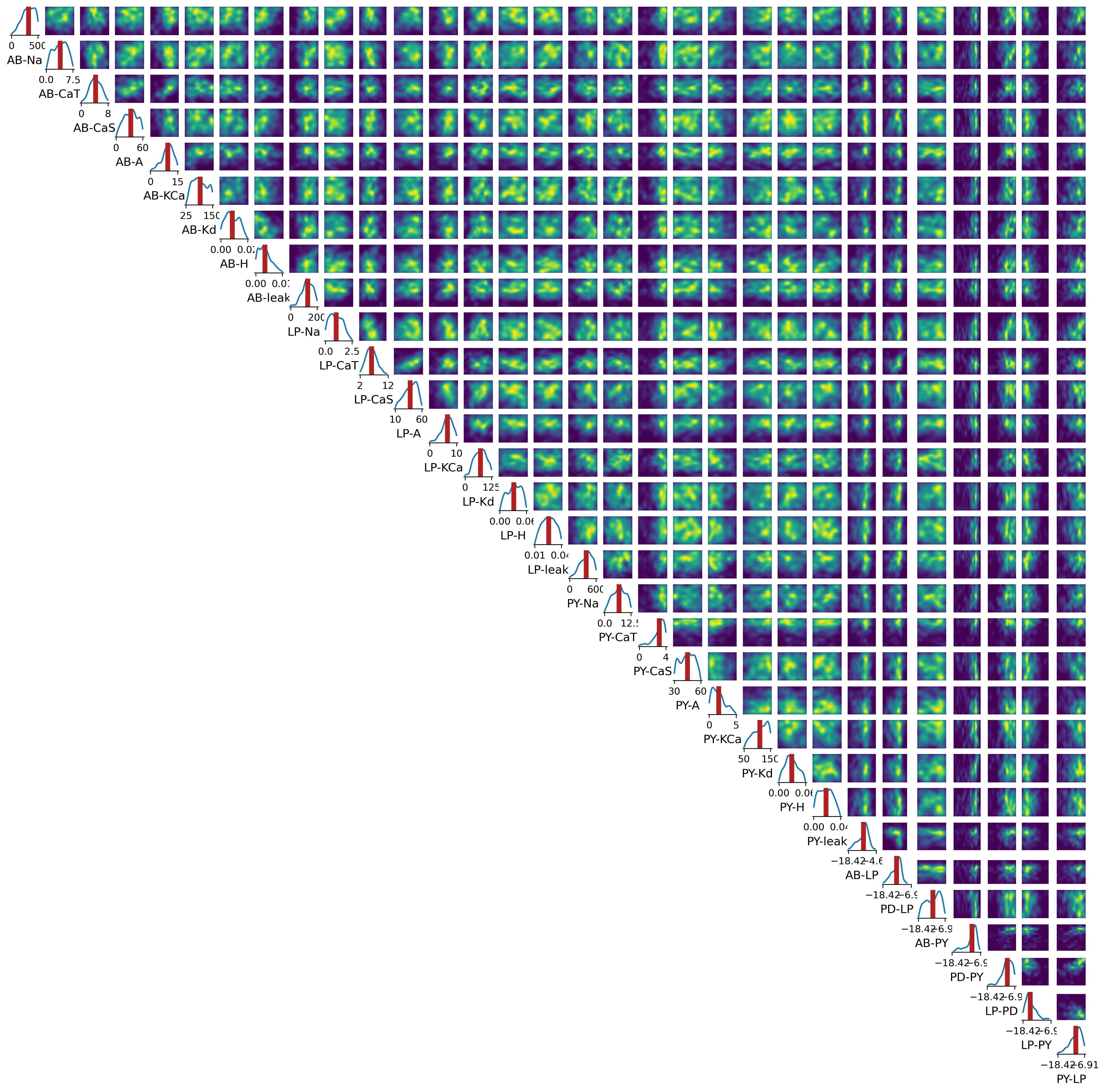}
	\caption{Pairwise marginals of SUNLE's posterior estimate on the \emph{C. borealis} simulator model.}
	\label{fig:pyloric-pairwise-marginals}
    \vspace{-1em}
\end{figure}

\vspace{-1em}

\paragraph{Use of a Calibration Network}
Due to the presence of invalid observations, we proceed as in \cite{glockler2021variational}  and fit a calibration network that allows to remove the bias induced by throwing away pairs of (parameters, observations) when the observations do not have well defined summary statistics. We use a similar architecture as in \cite{glockler2021variational}.

\end{document}